\documentclass[journal,twoside,web]{ieeecolor}
\usepackage{generic}
\usepackage{cite}
\usepackage{amsmath,amssymb,amsfonts}
\usepackage{algorithmic}
\usepackage{graphicx}
\usepackage{textcomp}
\usepackage{subfigure}
\usepackage{algorithm}
\newtheorem{theorem}{Theorem}
\newtheorem{proposition}{Proposition}

\newtheorem{lemma}{Lemma}
\newtheorem{remark}{Remark}
\newtheorem{definition}{Definition}
\newtheorem{assumption}{Assumption}

\definecolor{gray}{rgb}{0.5,0.5,0.5}
\def\BibTeX{{\rm B\kern-.05em{\sc i\kern-.025em b}\kern-.08em
		T\kern-.1667em\lower.7ex\hbox{E}\kern-.125emX}}
\begin{document}
	\title{Global Convergence of Online Identification for Mixed Linear Regression}
	\author{Yujing Liu, Zhixin Liu, \IEEEmembership{Member, IEEE} and Lei Guo, \IEEEmembership{Fellow, IEEE}
		\thanks{This work was supported by Natural Science Foundation of China under Grants T2293770 and 12288201, the National Key R\&D Program of China under Grant 2018YFA0703800, the Strategic Priority Research Program of Chinese Academy of Sciences under Grant No. XDA27000000.}
		\thanks{Yujing Liu, Zhixin Liu and Lei Guo are with the Key Laboratory of Systems and Control, Academy of Mathematics and Systems Science, Chinese Academy of Sciences, Beijing 100190, China, and also with the School of Mathematical Science, University of Chinese Academy of Sciences, Beijing 100049, China. (E-mails: liuyujing@amss.ac.cn, lzx@amss.ac.cn, Lguo@amss.ac.cn).}}
	
	\maketitle	
	\begin{abstract}
		Mixed linear regression (MLR) is a powerful model for characterizing nonlinear relationships by utilizing a mixture of linear regression sub-models. The identification of MLR is a fundamental problem, where most of the existing results focus on off-line algorithms, rely on independent and identically distributed (i.i.d) data assumptions, and provide local convergence results only. This paper investigates the online identification and data clustering problems for two basic classes of MLRs, by introducing two corresponding new online identification algorithms based on the expectation-maximization (EM) principle. It is shown that both algorithms will converge globally without resorting to the traditional i.i.d data assumptions. The main challenge in our investigation lies in the fact that the gradient of the maximum likelihood function does not have a unique zero, and a key step in our analysis is to establish the stability of the corresponding differential equation in order to apply the celebrated Ljung's ODE method. It is also shown that the within-cluster error and the probability that the new data is categorized into the correct cluster are asymptotically the same as those in the case of known parameters. Finally, numerical simulations are provided to verify the effectiveness of our online algorithms.		
	\end{abstract}
	
	\begin{IEEEkeywords}
		Mixed linear regression, online identification, non-i.i.d data, global convergence, data clustering
	\end{IEEEkeywords}
	
	\section{Introduction}\label{sec:introduction}
	Mixed linear regression (MLR) has been widely studied in the fields of system identification, statistical learning, and computer science due to its convenience and effectiveness in capturing the non-linearity of uncertain system dynamics.
	It was first proposed as a generalization of switching regressions \cite{quandt1978estimating} and has been found numerous applications including trajectory clustering \cite{gaffney1999trajectory}, health care analysis \cite{deb2000estimates}, face recognition \cite{chai2007locally}, population clustering \cite{ingrassia2014model}, drug sensitivity prediction \cite{li2019drug} and the relationship between genes and disease phenotype \cite{sun2022robust,chang2021supervised}.
	In MLR, each input-output data belongs to one of the uncertain linear regression models or sub-models but we don't know it comes from which sub-model, i.e., the label of data is unknown to us.	
	Thus, how to construct algorithms based on the observed data to identify the unknown parameters and to categorize the newly observed data into correct clusters is of fundamental importance for learning and prediction of the MLR. 
	Besides, this problem is closely related to the identification problem of piece-wise affine models (cf., \cite{bemporad2005bounded,nakada2005identification,paoletti2007identification}), where the mixing laws depend on the system states rather than on random processes as in the MLR.
	
	Though the MLR identification problem is NP-hard if there is no assumption imposed on the properties of the observed data \cite{yi2014alternating} , it still attracts much attention from researchers in diverse fields under additional assumptions such as Gaussian and the independent and identically distributed (i.i.d) assumptions on the data.
	The commonly used methods include the tensor-based method (cf., \cite{anandkumar2014tensor,chaganty2013spectral,sedghi2016provable}), the optimization-based method \cite{zhong2016mixed} and the expectation-maximization (EM) method \cite{dempster1977maximum}. 
	In the tensor-based method  (cf., \cite{pearson1894contributions}), an efficient spectral decomposition of the observed tensor matrix is needed so that the subspace spanned by true parameters is included in the subspace spanned by eigenvectors of the tensor matrix.
	By grid searching with a sufficiently small grid resolution in this subspace, the exact recovery guarantee for the MLR problem is given (see e.g., \cite{yi2014alternating,li2018learning}) , but this method suffers from high sample complexity and high computational complexity. 
	In the optimization-based method, minimizing the non-convex mean square error of the MLR problem can be converted into the optimization of some objective functions with nice properties such as convexity or smoothness \cite{zhong2016mixed}.
	However, solving the optimization problem is essentially to optimize a nuclear norm function with linear constraints \cite{chen2014convex} or to solve a mixed integer programming problem \cite{wang2019convergence}, both of which may lead to high computational cost.
	The EM algorithm \cite{dempster1977maximum}, including E-step and M-step, is a general technique to estimate unknown parameters with hidden random variables. 
	The E-step is used to evaluate the expectation of the log-likelihood function for the complete data set based on the current parameter estimate, while the M-step is used to update the estimate by solving the corresponding maximization problem.
	Compared with the other two methods, the lower computational cost of the EM algorithm makes it more attractive to solve MLR problems in practice (cf., \cite{gaffney1999trajectory,cohen1980influence}).
	
	In the theoretical aspect, there has been remarkable progress on solving the MLR identification problems by using the EM algorithm. 
	In the symmetric case where the parameters of the MLR model satisfy $\beta_1^*=-\beta_2^*$, this symmetric prior information can be used to simplify the design and analysis of EM algorithms (cf., \cite{netrapalli2013phase}).
	For example, Balakrishnan et al. \cite{balakrishnan2017statistical} studied the population EM algorithm and obtained local convergence results under the assumption that the regressor is i.i.d with a standard Gaussian distribution, and later Klusowski et al. \cite{klusowski2019estimating} proved a larger basin of attraction for local convergence.
	Kwon et al. \cite{kwon2019global} established the global convergence result of the population EM by verifying that both the angle and distance between the estimate and the true parameter are decreasing using Stein's Gaussian lemma under the i.i.d data assumptions, and thus overcame the difficulty of non-unique optimal parameters for MLR problems.
	For EM algorithm with finite number of samples, convergence results in the probability sense can also be found in (cf., \cite{kwon2019global,kwon2021minimax}).	
	In the general case where there is no symmetric assumptions on the MLR model, the local convergence of the population EM algorithm is established by verifying the convexity of a small neighborhood of the true parameters under the i.i.d Gaussian data assumption (cf., \cite{kwon2020converges, zhong2016mixed, klusowski2019estimating}).  
	Later, Zilber et al. proposed a novel EM-type algorithm in \cite {zilber2023imbalanced} and provided an upper bound on the estimation error with arbitrary initialization.
	There seems to be only one paper which is devoted to relax the i.i.d standard Gaussian assumption on the data \cite{li2018learning}, where the regressor is allowed to follow also the Gaussian distributions but with different covariances.	
	
	To summarize, all the above mentioned theoretical investigations have several common features. 
	Firstly, the regressors are required to be i.i.d with a Gaussian distribution, which is hard to be satisfied in many important situations, e.g., the forecasting of time series in economics \cite{hyndman2018forecasting}, multiple-point analysis in geostatistics \cite{honarkhah2010stochastic} and feedback control of stochastic uncertain systems \cite{lei2020feedback}.			
	Secondly, the computational algorithms are of off-line character.
	In fact, the off-line population EM algorithm is used in most previous investigations, which requires infinite number of samples at each iteration. The EM algorithm with finite number of samples has also been proposed (cf., \cite{balakrishnan2017statistical}, \cite{kwon2019global}), but the computational approach still remains off-line and the convergence results are derived in high probability sense only.
	In contrast to off-line algorithm, the online algorithm is desirable in many practical situations, which is updated conveniently based on both the current estimate and new input-output data, without requiring storage of all the old data and with lower computational cost.
	Thirdly, there is no global convergence results in general.
	The only exception is the symmetric MLR problem with the population EM algorithm where global convergence results are established in the i.i.d data case.
	For the general asymmetric MLR problem, only local convergence of the EM algorithm has been obtained and there is currently no theoretical guarantee of global convergence, even when adopting the population EM algorithm under i.i.d data Gaussian assumptions. 
	How to establish global convergence results without i.i.d assumptions for EM algorithms still remains to be an open problem. 
	
	In this paper, we propose online identification algorithms based on the EM principle for both symmetric and general asymmetric MLR problems, where the estimates are updated according to the current estimates and the new observations.
	We remark that Ljung's ODE method \cite{ljung1977analysis} stands out as a powerful analytic approach for general recursive algorithms without i.i.d data assumptions.
	By utilizing this method and making efforts to establish the stability of the corresponding ordinary differential equations (ODEs), we are able to provide global convergence results for parameter identification algorithms without resorting to the traditional i.i.d assumption on the regression data, based on which we can then analyze the performance of the data clustering algorithms.	
	The main contributions of this paper can be summarized as follows:
	\begin{itemize}
		\item 
		We first propose an online EM algorithm for the symmetric MLR problem, which alternates between computing the probability that new data belongs to each sub-model and updating parameter estimates.
		Based on this algorithm and the least squares (LS), we then study the general asymmetric MLR problem by devising a two-step online identification algorithm to estimate the unknown parameters of the asymmetric MLR model.
		\item	
		We establish the global convergence of the online EM algorithm for both the symmetric and general asymmetric MLR problems without the i.i.d data assumption.
		To the best of our knowledge, this paper is the first to establish a convergence theory for the online identification of  MLR models.
		\item
		Based on the proposed new online identification algorithms, we prove that the performance of data clustering including the within-cluster error and the probability that the new data can be categorized into the correct cluster can asymptotically achieve the same performance as those in the case where true parameters are known.
	\end{itemize}
	
	The remainder of this paper is organized as follows. Section \ref{pf} presents the problem formulation. In Section \ref{Algorithms}, we propose our online EM algorithms. Sections \ref{mainresults} states the main results on the global convergence of parameter identification and the performance of data clustering algorithms for MLR problems. Sections \ref{proof} gives the proofs of the main results. Section \ref{simu} provides numerical simulations to verify the effectiveness of our algorithms. Finally, we conclude the paper in Section \ref{conclusion}.
	
	\section{Problem Formulation}\label{pf}
	\subsection{Basic Notations}
	In the sequel, $v\in{\mathbb{R}^d}$ represents a $d$-dimensional column vector, ${v^{\tau}}$ and $\|v\|$ denote its transpose and Euclidean norm, respectively.
	For a $d\times d$-dimensional matrix $A$, $\|A\|$ denotes its Euclidean norm, $tr(A)$ is the trace, and the maximum and minimum eigenvalues are denoted by $\lambda_{max}(A)$ and $\lambda_{min}(A)$, respectively.
	For two matrices $A$ and $B$, $A>(\geq)B$ means that $A-B$ is a positive (semi-positive)-definite matrix.
	
	Let $(\Omega, \mathcal{F},P)$ be a probability space, where $\Omega$ is the sample space, the $\sigma$-algebra $\mathcal{F}$ on $\Omega$ is a family of events and $P$ is a probability measure on $(\Omega, \mathcal{F})$.
	For an event $\mathcal{A}\in \mathcal{F}$, its complement $\mathcal{A}^c$ is defined by $\mathcal{A}^c=\Omega-\mathcal{A}$.
	The indicator function $\mathbb{I}_\mathcal{A}$ on $\Omega$ is defined by $\mathbb{I}_\mathcal{A}=1$ if the event $\mathcal{A}$ occurs and $\mathbb{I}_\mathcal{A}=0$ otherwise.
	If $P(\mathcal{A})=1$, then it is said that the event $\mathcal{A}$ occurs almost surely (a.s.).
	An infinite sequence of events $\{\mathcal{A}_k,k\geq0\}$ is said to happen infinitely often (i.o.) if $\mathcal{A}_k$ happens for an infinite number of indices $k\in\{1,2,\cdots\}$.
	Moreover, a sequence of random variables $\{x_k,k\geq0\}$ is called uniformly intergrable (u.i.) if $\lim\limits_{a\to\infty}\sup_{k\geq1}\int_{[|x_k|>a]}|x_k|dP=0$.
	We use $\mathbb{E}\left[\cdot \right]$ to denote the mathematical expectation operator, and $\mathbb{E}[\cdot|\mathcal{F}_k]$ to represent the conditional expectation operator given $\mathcal{F}_k$, where $\{\mathcal{F}_k\}$ is a non-decreasing sequence of $\sigma$-algebras.
	According to convention, $x\sim F$ indicates that the random variable $x$ obeys the distribution $F$ and $\mathcal{N}(\mu,\sigma^2)$ represents the Gaussian distribution with $\mu$ and $\sigma^2$ being the mean and the variance, respectively.
	
	\begin{definition}
		A sequence of random variables $\{x_k,k\geq1\}$ is said to be asymptotically stationary  if for any $\epsilon>0$, and any set $C\in\mathcal{B}^{\infty}$ with $\mathcal{B}^{\infty}$ being the Borel set of $\mathbb{R}^{\infty}$, there exists $K>0$ such that for all $k\geq K$,
		$$|P(\{x_k,x_{k+1},\cdots\}\in C)-P(\{x_{k+1},x_{k+2},\cdots\}\in C)|\leq\epsilon.$$
		It is further ergodic if $\lim\limits_{k\to\infty}\mathbb{E}\|x_k\|$ exists and
		$$\lim\limits_{n\to\infty}\frac{1}{n}\sum\limits_{k=1}^{n}x_k=\lim\limits_{k\to\infty}\mathbb{E}x_k.$$
	\end{definition}
	
	In the above definition, if both parameters $\epsilon$ and $K$ can take value $0$, then the sequence $\{x_k,k\geq1\}$ is called stationary and ergodic, which is consistent with the traditional definition as in \cite{stout1974almost}.
	\subsection{Problem Statement}
	Consider the following mixed linear regression (MLR) model consisting of two sub-models:
	\begin{equation}\label{model}
		y_{k+1}=
		\left\{\begin{aligned}
			&\beta_1^{*\tau}\phi_{k}+w_{k+1},~~\text{if}~~z_k=1,\\
			&\beta_2^{*\tau}\phi_{k}+w_{k+1},~~\text{if}~~z_k=-1,
		\end{aligned}\right.		
	\end{equation}
	where $\beta_1^{*}$ and $\beta_2^{*}$ are unknown parameter vectors in $\mathbb{R}^d$ that determine the sub-models, $\phi_k\in\mathbb{R}^d$, $y_{k+1}\in\mathbb{R}$ and $w_{k+1}\in\mathbb{R}$ are the regressor vector, observation and the system noise. In addition, $z_k\in\{-1,1\}$ is a hidden variable, namely, we do not know which sub-model the data $\{\phi_k,y_{k+1}\}$ comes from.	
	
	\begin{remark}
		The presence of multiple linear regressions makes the MLR model highly expressive and thus it has a wide range of applications, e.g., movements of object in video sequence \cite{gaffney1999trajectory}, medical insurance cost \cite{deb2000estimates}, human perception of tones \cite{cohen1980influence} and World Health Organization (WHO) life expectancy \cite{zilber2023imbalanced}.
		It is worth mentioning that the MLR model is said symmetric if the true parameters satisfy $\beta_1^*=-\beta_2^*$ and the phase retrieval model \cite{netrapalli2013phase} is such a case. 
		The MLR model is said balanced if the hidden variable $z_k$ has equal probabilities, i.e., $P(z_k=1)=P(z_k=-1)$ and the human perception of tones \cite{cohen1980influence} is such a case.
	\end{remark}
	
	It is worth mentioning that under some excitation conditions on the data, the optimal solution of
	$$
	J_n=\min_{\beta_1,\beta_2\in\mathbb{R}^d}\frac{1}{n}\sum\limits_{k=1}^n\min_{i=1,2}\left\{(y_{k+1}-\beta_i^{\tau}\phi_k)^2\right\}
	$$ 
	will converge to the true parameters $\beta_1^*, \beta_2^*\in\mathbb{R}^d$ when $n\to\infty$ as shown by Wang et al. \cite{wang2019convergence}.
	However, it is hard in general to find a computational algorithm for solving this mixed optimization problem.
	
	The aim of this paper is to develop online algorithms to simultaneously estimate the true parameters $\beta_1^*$ and $\beta_2^*$ by using the streaming data $\{\phi_k,y_{k+1}\}_{k=1}^{\infty}$, and establish convergence results for the identification algorithms.
	Based on the estimates of $\beta_1^*$ and $\beta_2^*$, we further investigate the probability and performance that the newly generated data can be categorized into correct clusters.
	
	\section{Online EM Algorithms}\label{Algorithms}
	
	In this section, we design online identification algorithms based on the likelihood method for both symmetric and general asymmetric MLR models.
	We note that the symmetric MLR case will bring benefits for the design and theoretical analysis of the algorithm and will pave a way for the study of the general asymmetric MLR case, we will therefore first consider the identification problem of the symmetric MLR problem.
	
	\subsection{Online EM Algorithm for Symmetric MLR Problem}	
	The symmetric MLR model can be simplified as follows:
	\begin{equation}\label{model2}
		y_{k+1}=z_k\beta^{*\tau}\phi_{k}+w_{k+1},	
	\end{equation}
	where $\beta_1^*$ in (\ref{model}) is denoted as $\beta^*$. The symmetric prior information simplifies the estimation of the posterior probability of which cluster the newly generated data comes from, thereby facilitating the design and analysis of the algorithm.
	
	To start with, let us first assume that the noise $\{w_{k+1}\}$ is i.i.d with a normal distribution $\mathcal{N}(0,\sigma^2)$, the hidden variable $\{z_k\}$ is i.i.d and balanced, the regressor $\{\phi_k\}$ is i.i.d and also $\{z_k\}$, $\{\phi_k\}$ and $\{w_{k+1}\}$ are mutually independent\footnote{Note that the balanced assumption and the independence assumption on the data will actually not be used in both the theorem and the analysis.}. Then we derive the likelihood function of the MLR model (\ref{model2}) with the parameter $\beta$ as follows:
	\begin{equation}\label{ml}
		\begin{aligned}
			\mathcal{L}_n(\beta)&=P(\mathbb{O}_n|\beta,\mathbb{U}_n)=\prod\limits_{k = 1}^n {P(y_{k+1}|\beta,\phi_k)}\\
			&=\prod\limits_{k = 1}^n\bigg\{\frac{1}{2\sqrt{2\pi\sigma^2}}\exp
			\left(-\frac{(y_{k+1}-\beta^{\tau}\phi_k)^2}{2\sigma^2}\right)\\
			&~~+\frac{1}{2\sqrt{2\pi\sigma^2}}\exp
			\left(-\frac{(y_{k+1}+\beta^{\tau}\phi_k)^2}{2\sigma^2}\right)\bigg\},
		\end{aligned}
	\end{equation}
	where $\mathbb{U}_n=\{\phi_1,\cdots,\phi_n\}$ and $\mathbb{O}_n=\{y_1,\cdots,y_{n+1}\}$.  
	
	With simple calculations, the gradient of the corresponding log-likelihood function with respect to $\beta$ has the following expression:
	\begin{equation}\label{grad}
		\begin{aligned}
			&Q_n(\beta)=\nabla_{\beta} \log(\mathcal{L}_n)\\
			=&\frac{1}{\sigma^2}\bigg\{\sum\limits_{k=1}^n\left[-\phi_k\phi_k^{\tau}\beta+\phi_ky_{k+1}\tanh\left(\frac{\beta^{\tau}\phi_ky_{k+1}}{\sigma^2}\right)\right]\bigg\},
		\end{aligned}		
	\end{equation}
	where $\tanh(x)$ is the hyperbolic tangent function defined as $\tanh(x)=\frac{\exp(x)-\exp(-x)}{\exp(x)+\exp(-x)}$. 
	We can see that 0 is a local minimum, and it is hard to obtain the closed-form expression of maximum likelihood estimation (MLE).  
	Hence, we adopt the EM algorithm (cf., \cite{dempster1977maximum}) to approximate the MLE.		
	Denote $\beta_t$ as the estimates of $\beta^*$ at the time instant $t$.
	The EM algorithm is conducted according to two steps:
	
	1) E-step: compute an auxiliary function $\mathcal{Q}_k$, i.e., the log-likelihood for the data set $\{y_{t+1},z_t,\phi_t\}_{t=1}^k$ based on $\beta_t$:
	$$
	\begin{aligned}
		&\mathcal{Q}_k(\beta)=-\frac{1}{2\sigma^2}\bigg\{\sum\limits_{t=1}^k\bigg[P(z_t=1|\phi_t,y_{t+1},\beta_t)(y_{t+1}-\beta^{\tau}\phi_t)^2\\
		&~~~+P(z_t=-1|\phi_t,y_{t+1},\beta_t)(y_{t+1}+\beta^{\tau}\phi_t)^2\bigg]\bigg\}+kc,
	\end{aligned}
	$$
	where $c=\log(\frac{1}{2\sqrt{2\pi}\sigma})$, the conditional probabilities of the hidden variable $z_t$ given $\{\phi_t, y_{t+1}, \beta_t\}$ are as follows:
	$$
	\begin{aligned}
		&P(z_t=1|\phi_t,y_{t+1},\beta_t)\\	=&\frac{\exp\left(-\frac{(y_{t+1}-\beta_t^{\tau}\phi_t)^2}{2\sigma^2}\right)}{\exp\left(-\frac{(y_{t+1}-\beta_t^{\tau}\phi_t)^2}{2\sigma^2}\right)+\exp\left(-\frac{(y_{t+1}+\beta_t^{\tau}\phi_t)^2}{2\sigma^2}\right)}\\
		=&\frac{\exp\left(\frac{\beta_t^{\tau}\phi_ty_{t+1}}{\sigma^2}\right)}{\exp\left(\frac{\beta_t^{\tau}\phi_ty_{t+1}}{\sigma^2}\right)+\exp\left(-\frac{\beta_t^{\tau}\phi_ty_{t+1}}{\sigma^2}\right)},
	\end{aligned}
	$$
	and
	$$
	\begin{aligned}
		&P(z_t=-1|\phi_t,y_{t+1},\beta_t)=1-P(z_t=1|\phi_t,y_{t+1},\beta_t).\\
	\end{aligned}$$
	
	2) M-step: update the estimate $\beta_k$ by maximizing the function $\mathcal{Q}_k(\beta)$:
	\begin{equation}\label{offline}
		\begin{aligned}
			&\beta_{k+1}=\mathop{\arg\max}_{\beta}\mathcal{Q}_k(\beta)\\
			=&\bigg(\sum\limits_{t=1}^k\phi_t\phi_t^{\tau}\bigg)^{-1}\bigg(\sum\limits_{t=1}^k\phi_ty_{t+1}\tanh\bigg(\frac{\beta_t^{\tau}\phi_ty_{t+1}}{\sigma^2}\bigg)\bigg).
		\end{aligned}
	\end{equation}
	It is clear that the EM algorithm is a soft version of the well-known $k$-means algorithm \cite{lloyd1982least}.
	
	Denote 
	\begin{equation}\label{bary}
		\bar y_{k+1}=y_{k+1}\tanh\big(\frac{\beta_k^{\tau}\phi_ky_{k+1}}{\sigma^2}\big).
	\end{equation}
	We can see that the equation (\ref{offline}) is the standard formula of LS with output $\bar y_{k+1}$ and the regressor $\phi_k$, hence following the same way as the derivation of the recursive LS \cite{chen1991identification}, we get the resulting online EM algorithm as shown in Algorithm \ref{alg1}.
	
	\begin{algorithm}
		\caption{Online EM algorithm for symmetric MLR problem}
		\label{alg1}
		\begin{algorithmic}[1]
			\STATE Initialization: $\beta_{0}\ne0$, $P_0>0$.
			\STATE At each time step $k+1$, we have data $\{\phi_k,y_{k+1}\}$.
			\STATE Recursively calculate the estimate:
			\begin{footnotesize}\begin{subequations}\begin{align}
						&\beta_{k+1}={\beta}_{k}+a_kP_k{\phi}_{k}\bigg(y_{k+1}\tanh\bigg(\frac{\beta_k^{\tau}\phi_ky_{k+1}}{\sigma^2}\bigg)-{\beta}_{k}^{\tau}{\phi}_{k}\bigg),\label{beta}\\
						&P_{k+1}=P_k-a_kP_k\phi_k\phi_k^{\tau}P_k,\label{pk}\\ &a_k=\frac{1}{1+\phi_k^{\tau}P_k\phi_k}.
			\end{align}\end{subequations}\end{footnotesize}
			\STATE Output: $\beta_{k+1}$		
		\end{algorithmic}
	\end{algorithm}	
	\subsection{Online EM Algorithm for General Asymmetric MLR Problem}
	First of all, we show that the general asymmetric case can be transferred to a case that can be dealt with as the symmetric case together with the case of LS.
	
	Let us denote $\theta_1^*=\frac{\beta_1^*+\beta_2^*}{2}$ and $\theta_2^*=\frac{\beta_1^*-\beta_2^*}{2}$. 
	Clearly, the parameters $\beta_1^*$ and $\beta_2^*$ in (\ref{model}) will be obtained once the parameters $\theta_1^*$ and $\theta_2^*$ are identified.
	The MLR model (\ref{model}) can then be equivalently rewritten into the following model:
	\begin{equation}\label{model1}
		y_{k+1}=\theta_1^{*\tau}\phi_k+z_k\theta_2^{*\tau}\phi_k+w_{k+1}.
	\end{equation}
	Under the balanced assumption on the hidden variable $z_k$, and the mutually independence assumption among $z_k$, $\phi_k$ and $w_{k+1}$ for each $k\geq 0$, we have $\mathbb{E}\left[z_k\theta_2^{*\tau}\phi_k|\phi_k\right]=0$, and then
	\begin{equation}\label{bm}
		\mathbb{E}\left[y_{k+1}|\phi_k\right]=\theta_1^{*\tau}\phi_k,
	\end{equation}
	which is actually a linear regression model and can be estimated by the LS algorithm.
	
	Thus, we propose a two-step identification algorithm to estimate the parameters $\theta_1^*$ and $\theta_2^*$.
	Firstly, by (\ref{bm}), the parameter $\theta_1^*$ is estimated using the LS algorithm. Secondly, by replacing the unknown parameter $\theta_1^*$ in (\ref{model1}) by its LS estimate given in the first step, we can then estimate the unknown parameter $\theta_2^*$ in the same way as that for the symmetric case stated above. 
	The whole algorithm is summarized in the following Algorithm \ref{alg2}:
	
\begin{algorithm}
	\caption{Online EM algorithm for general asymmetric MLR problem}
	\label{alg2}
	\begin{algorithmic}[1]
		\STATE Initialization: $\theta_{0,1}$, $\theta_{0,2}\ne0$, $P_0>0$.
		\STATE At each time step $k+1$, we have data $\{\phi_k,y_{k+1}\}$.
		\STATE Recursively calculate the estimate:
		\begin{footnotesize}
		$$
		\begin{aligned}
		&\textbf{\#Step 1: LS~ estimation~for~the~paprameter~$\theta_1^*$}\\
		&\theta_{k+1,1}={\theta}_{k,1}+a_kP_k{\phi}_{k}\left(y_{k+1}-{\theta}_{k,1}^{\tau}{\phi}_{k}\right),\\
\		&P_{k+1}=P_k-a_kP_k\phi_k\phi_k^{\tau}P_k,~~a_k=\frac{1}{1+\phi_k^{\tau}P_k\phi_k},	\\
		&\textbf{\#Step 2: EM~ estimation~for~the~paprameter~$\theta_2^*$}\\
		&\theta_{k+1,2}={\theta}_{k,2}+a_kP_k\phi_k\bigg[m_{k+1}\tanh\big(\frac{\theta_{k,2}^{\tau}\phi_km_{k+1}}{\sigma^2}\big)-{\theta}_{k,2}^{\tau}{\phi}_{k}\bigg],\\
		&m_{k+1}=y_{k+1}-{\theta}_{k,1}^{\tau}{\phi}_{k}.	
		\end{aligned}
		$$
		\end{footnotesize}
		\STATE Outputs:
		$$\beta_{k+1,1}=\theta_{k+1,1}+\theta_{k+1,2},\beta_{k+1,2}=\theta_{k+1,1}-\theta_{k+1,2}.$$
	\end{algorithmic}
\end{algorithm}
	
	\section{Main Results}\label{mainresults}	
	In this section, we give the main results concerning the convergence of parameter estimates and the data clustering performance of Algorithms \ref{alg1} and \ref{alg2}, respectively.
	\subsection{Performance of Algorithm \ref{alg1}}\label{AnalysisAlgorithm1}	
	To establish a rigous theory on the performance of the identification algorithms, we need to introduce some assumptions on the hidden variable $z_k$, the noise $w_{k+1}$ and the regressor $\phi_k$.
	\begin{assumption}\label{asm2}
		The sequence of hidden variables $\{z_k\}$ is i.i.d with distribution $P(z_k=1)=p\in(0,1)$ and $P(z_k=-1)=1-p$. In addition, $z_k$ is independent of $\phi_k$ for each $k\geq0$.
	\end{assumption}
	\begin{remark}
		In most of the existing results on the symmetric MLR (e.g., \cite{kwon2019global}), the balanced mixture (i.e., $p=\frac{1}{2}$) is assumed in the convergence analysis. The above assumption is more general as it includes both the balanced and unbalanced mixtures.
	\end{remark}
	\begin{assumption}\label{asm3} 
		The noise $\{w_{k+1}\}$ is a sequence of i.i.d random variables with Gaussian distribution $\mathcal{N}(0,\sigma^2)$. 
		In addition, $w_{k+1}$ is independent of $\{z_t\}_{t\leq k}$ and $\{\phi_t\}_{t\leq k}$ for $k\geq0$.
	\end{assumption}
	\begin{assumption}\label{asm4}
		The regressor sequence $\{\phi_k\}$ is asymptotically stationary and ergodic and $\{\|\phi_k\|^4\}$ is uniformly intergrable. In addition, its probability density function (p.d.f) $g_k(x)$ satisfies
		\begin{equation}\label{barg}
			\begin{aligned}
				&\lim\limits_{k\to\infty} g_k(x)=\bar g(x)\in\mathcal{G}=\big\{g(x): g(x)~\text{is a function }\\
				&\hskip 1.5cm\text{of}~\|x\|~ \text{for}~x\in\mathbb{R}^d, \int{xx^{\tau}g(x)dx}>0\big\}.
			\end{aligned}
		\end{equation}
	\end{assumption}
	\begin{remark}
		We remark that the set $\mathcal{G}$ of probability density functions include many familiar distributions, such as the uniform distribution on the sphere, the isotropic Gaussian distribution, the Logistic distribution, the polynomial distribution and the probability density functions introduced in \cite{qian2019global}.
	\end{remark}
	\begin{remark}\label{nonrato}
		The requirement that $g(x)$ is a function of $\|x\|$ means that it has the rotation-invariant property.
		This property is assumed for simplicity of presentation and can be further relaxed.
		For example, if the asymptotically stationary density function of $\phi_k$ is Gaussian with zero mean and non-unit covariance matrix $\Sigma>0$, then $g(x)$ will have the form $g_0(\Sigma^{-1/2}x)$, which does not satisfy the rotation-invariant property although the standard normal density $g_0(x)$ does.
		In this case, let 
		$$\begin{aligned}
			\widehat\phi_k=\bar R_{k}^{-\frac{1}{2}}\phi_k, \bar R_{k}=\bar R_{k-1}+\frac{1}{k}\left(\phi_k\phi_k^{\tau}-\bar R_{k-1}\right).
		\end{aligned}$$
		From the facts that $\bar R_{k}\to\Sigma$ almost surely as $k\to\infty$, it is easy to obtain that $\widehat\phi_k$ converges in distribution to a standard normal random variable. 		
		Then by replacing $\phi_k$ with $\widehat\phi_k$ in Algorithm \ref{alg1}, it can be transferred to the rotation invariant case in Assumption \ref{asm4} and our main results in Theorems \ref{odetheorem1}-\ref{odetheorem4} still hold.
	\end{remark}
	
	\begin{remark}
		Assumption \ref{asm4} is weaker than that used in most of the existing literature where the regressor $\{\phi_k\}$ is required to be i.i.d with a standard Gaussian distribution (cf., \cite{kwon2019global,klusowski2019estimating,balakrishnan2017statistical}).
		In particular, with the modification technique in Remark \ref{nonrato}, our analysis can contain the case where $\{\phi_k\}$ is generated by the following dynamic system:
		$$\phi_{k+1}=A\phi_{k}+e_{k+1},$$
		where $e_k$ is i.i.d with Gaussian distribution and $\bar\rho(A)<1$ with $\bar\rho(A)$ being the spectral radius of $A$.
	\end{remark}

	Based on the above assumptions, we give the convergence result for parameter identification and data clustering performance of Algorithm \ref{alg1} as follows:
	\subsubsection{Convergence of Algorithm \ref{alg1}}	
	 We give the following main theorem on the convergence of the identification Algorithm \ref{alg1}:
	\begin{theorem}\label{odetheorem1}
		Let Assumptions \ref{asm2}-\ref{asm4} be satisfied. Then for any initial values $\beta_0\ne0$ and $P_0>0$, the estimate $\beta_k$ generated by Algorithm \ref{alg1} will converge to a limit point that belongs to the set $\{\beta^*, -\beta^*\}$ almost surely.
	\end{theorem}
	\begin{remark}
		Note that the convergence property provided in Theorem \ref{odetheorem1} is of local nature in the sense that the limit point of $\beta_k$ may depend on the initial value $\beta_0$ of Algorithm \ref{alg1}.
		A somewhat surprising fact is that this local convergence property is sufficient for guaranteeing the global optimality of the data clustering performance asymptotically, as will be rigorously shown in Theorem \ref{dcp} below.
		This is reminiscent of the well-known self-tuning regulators in adaptive control, where the control performance can still achieve its optimal value even though the parameter estimates may not converge to the true parameter values (cf.,\cite{aastrom1973self,becker1985adaptive,chen1991identification}) .
	\end{remark}
	
	\subsubsection{Clustering Performance of Algorithm \ref{alg1}}\label{ClusteringPerformanceofAlgorithm1}
	Based on the estimate $\beta_k$ generated by Algorithm \ref{alg1}, we can online categorize the new data $\{\phi_k,y_{k+1}\}$ to the corresponding cluster $\mathcal{I}_k\in\{1,2\}$ according to the following criterion:
	\begin{equation}\label{class}
		\mathcal{I}_k=\mathop{\arg\min}_{i=1,2}\{(y_{k+1}-(-1)^{i}\beta_{k}^{\tau}\phi_k)^2\}.
	\end{equation}
	To evaluate the clustering performance for the within-cluster errors, a commonly-used evaluation index (cf., \cite{zilber2023imbalanced}) is defined as follows:
	\begin{equation}\label{wce}
		J_n=\frac{1}{n}\sum\limits_{k=1}^n(y_{k+1}-(-1)^{\mathcal{I}_k}\beta_{k}^{\tau}\phi_k)^2.
	\end{equation}
	
	Our purpose is to provide a lower bound to the probability that $\{y_{k+1},\phi_k\}$ can be categorized into the correct cluster, and an upper bound of the within-cluster error.
	In general, the regressor $\phi_k$ may contain feedback control input signals, which may depend on the estimate $\beta_k$ if the input is an adaptive control designed based on the current estimate $\beta_k$. 
	Thus the p.d.f of $\phi_k$ may depend on $\beta_k$ in general.
	In the case, we assume that the conditional p.d.f of $\phi_k$ given $\beta_k$ is equi-continuous and convergent at the point $\beta_k=\beta^*(-\beta^{*})$ in the data clustering analysis of the paper.
	The main result on the performance of data clustering is stated as follows:	 		
	\begin{theorem}\label{dcp}
		Let Assumptions \ref{asm2}-\ref{asm4} be satisfied. Then the probability that the new data $\{\phi_k, y_{k+1}\}$ is categorized into the correct cluster is bounded from below by 
		\begin{equation}\label{them4.21}
			\begin{aligned}
				&\lim\limits_{k\to\infty}P(\{\phi_k,y_{k+1}\}~\text{is categorized correctly based on} \\
				&~~~~~~~~~~~~~~~~~~~~~~~~~~~~~~~~~~~\text{the estimate $\beta_k$})\\
				\geq&1-\mathbb{E}\left[\exp\left(-\frac{(\beta^{*\tau}\phi)^2}{2\sigma^2}\right)\right],
			\end{aligned}			
		\end{equation}
		and the within-cluster error (\ref{wce}) satisfies
		\begin{equation}\label{them4.22}
			\lim\limits_{n\to\infty}J_n =\sigma^2+4\mathbb{E}\left[\eta(\phi)\right]\leq\sigma^2,
		\end{equation}
		with
		$$\eta(\phi)=(\beta^{*\tau}\phi)^2\Phi\left(-\frac{|\beta^{*\tau}\phi|}{\sigma}\right)-\sigma|\beta^{*\tau}\phi|\Phi'\left(-\frac{|\beta^{*\tau}\phi|}{\sigma}\right)\leq0,$$
		where $\phi$ is a random vector with p.d.f $\bar g\in\mathcal{G}$ being the asymptotic stationary p.d.f of $\phi_k$, $\Phi(x)$ and $\Phi'(x)$ are the standard Gaussian distribution function and density function, respectively.
	\end{theorem}
	\begin{remark}
		From the proof of Theorem \ref{dcp}, one can find that the bounds given in (\ref{them4.21}) and (\ref{them4.22}) are actually the same bounds as in the case where the true parameter $\beta^*$ is known.
		It can also be seen that the data clustering performance is positively related to the signal-to-noise ratio $\frac{|\beta^{*\tau}\phi|}{\sigma}$. 
		Specifically, as the signal-to-noise ratio tends to infinity, the lower bound of probability of correct categorization will tend to $1$, while the upper bound of the within-cluster error will tend to $\sigma^2$.
	\end{remark}
	
	It goes without saying that given a specific form of the density function $\bar g$, one can obtain a more explicit bound concerning the probability that the new data is categorized into the correct cluster.
	
	$\textbf{Example 1:}$
	Let conditions of Theorem \ref{dcp} be satisfied and $\bar g$ be the p.d.f of Gaussian distribution $\mathcal{N}(0,\Sigma)$ with $\Sigma>0$. Then with the estimate $\beta_k$ generated by Algorithm \ref{alg1}, we have
	\begin{equation}\label{them4.23}
		\begin{aligned}
			&\lim\limits_{k\to\infty}P(\{\phi_k,y_{k+1}\}~\text{is categorized correctly based on}\\ &~~~~~~~~~\text{the estimate $\beta_k$})\geq1-\frac{1}{\sqrt{1+\frac{\beta^{*\tau}\Sigma\beta^*}{\sigma^2}}}.
		\end{aligned}
	\end{equation}
	The proof of (\ref{them4.23}) is provided in Appendix \ref{app1}.
	
	\subsection{Performance of Algorithm \ref{alg2}}\label{AnalysisAlgorithm2}
	In this subsection, we first give the convergence results of Algorithm \ref{alg2} and then present the corresponding data clustering performance.
	For this purpose, the noise $\{w_{k+1}\}$ and the regressor $\{\phi_k\}$ are still assumed to obey Assumptions \ref{asm3} and \ref{asm4} in Section \ref{AnalysisAlgorithm1}, while the sequence of hidden variables $\{z_k\}$ is assumed to satisfy the following assumption:
	
	\textit{Assumption \ref{asm2}$'$}:
	The sequence of hidden variables $\{z_k\}$ is i.i.d with balanced distribution $P(z_k=1)=P(z_k=-1)=\frac{1}{2}$. In addition, $z_k$ is independent of $\phi_k$ for each $k\geq 0$.
	\subsubsection{Convergence of Algorithm \ref{alg2}}
	
	Based on the convergence theory presented in Section \ref{AnalysisAlgorithm1} and the celebrated convergence property of the LS, we can obtain the following main result on the convergence of Algorithm \ref{alg2}:
	\begin{theorem}\label{odetheorem2}
		Let Assumptions \ref{asm2}$'$ and \ref{asm3}-\ref{asm4} be satisfied. Then for any initial values $\theta_{0,1}$, $\theta_{0,2}\ne0$ and $P_0>0$, the estimate $\left(\beta_{k,1},\beta_{k,2}\right)$ generated by Algorithm \ref{alg2} will converge to a limit point that belongs to the set $\left\{\left(\beta_1^*,\beta_2^*\right),\left(\beta_2^*,\beta_1^*\right)\right\}$ almost surely.
	\end{theorem}
	\subsubsection{Clustering Performance of Algorithm \ref{alg2}}
	Similar to the analysis in Section \ref{AnalysisAlgorithm1}, based on the estimates $\beta_{k,1}$ and $\beta_{k,2}$ generated by Algorithm \ref{alg2}, we can also online categorize $\{\phi_k,y_{k+1}\}$ to the corresponding cluster $\mathcal{I}'_k\in\{1,2\}$ according to the following criterion:
	\begin{equation}\label{class1}
		\mathcal{I}'_k=\mathop{\arg\min}_{i=1,2}\{(y_{k+1}-\beta_{k,i}^{\tau}\phi_k)^2\}.
	\end{equation}
	Moreover, the corresponding within-cluster error is defined as
	\begin{equation}\label{wce1}
		J_n'=\frac{1}{n}\sum\limits_{k=1}^n(y_{k+1}-\beta_{k,\mathcal{I}'_k}^{\tau}\phi_k)^2.
	\end{equation}
	
	In the following theorem, we give an asymptotic lower bound to the probability that $\{y_{k+1},\phi_k\}$ can be categorized correctly and an upper bound of the within-cluster error:
	\begin{theorem}\label{odetheorem4}
		Let Assumptions \ref{asm2}$'$ and \ref{asm3}-\ref{asm4} be satisfied.
		Then the probability that the new data $\{\phi_k, y_{k+1}\}$ is categorized into the correct cluster is bounded from below by
		$$
		\begin{aligned}
			&\lim\limits_{k\to\infty}P(\{\phi_k,y_{k+1}\}~\text{is categorized correctly based on} \\
			&~~~~~~~~~~~~~~~~~~~~~~~~~~\text{the estimates $\beta_{k,1}$, $\beta_{k,2}$})\\
			\geq&1-\mathbb{E}\left[\exp\left(-\frac{((\beta_1^*-\beta_2^*)^{\tau}\phi)^2}{8\sigma^2}\right)\right],
		\end{aligned}		
		$$
		and the within-cluster error (\ref{wce1}) satisfies
		$$
		\lim\limits_{n\to\infty}J_n'=\sigma^2+\mathbb{E}\left[\eta(\phi)\right]\leq\sigma^2,
		$$
		where $\eta(\phi)=((\beta_1^*-\beta_2^*)^{\tau}\phi)^2\Phi\left(-\frac{|(\beta_1^*-\beta_2^*)^{\tau}\phi|}{2\sigma}\right)-2\sigma|(\beta_1^*-\beta_2^*)^{\tau}\phi|\Phi'\left(-\frac{|(\beta_1^*-\beta_2^*)^{\tau}\phi|}{2\sigma}\right)$, $\phi$, $\Phi(x)$ and $\Phi'(x)$ are defined in Theorem \ref{dcp}.
	\end{theorem}
	
	Similar to Example 1, if $\bar g$ has a specific form, one can also obtain the following concrete result:
	
	$\textbf{Example 2:}$
	Let conditions of Theorem \ref{odetheorem4} be satisfied and $\bar g$ be defined in Example 1.
	Then with the estimates $\beta_{k,1}$ and $\beta_{k,2}$ generated by Algorithm \ref{alg2}, we have
	$$
	\begin{aligned}
		&\lim\limits_{k\to\infty}P(\{\phi_k,y_{k+1}\} ~\text{is categorized correctly based on the}\\
		&~~~~~\text{estimates $\beta_{k,1}$, $\beta_{k,2}$})\geq1-\frac{1}{\sqrt{1+\frac{(\beta_1^*-\beta_2^*)^{\tau}\Sigma(\beta_1^*-\beta_2^*)}{4\sigma^2}}}.
	\end{aligned}
	$$
	\section{Proofs of the Main Results}\label{proof}
	In this section, we provide proofs of Theorems \ref{odetheorem1}-\ref{odetheorem4}.
	\subsection{Proof of Theorem 1}
	The celebrated Ljung's ODE method \cite{ljung1977analysis} provides a general analytical technique for recursive algorithms by establishing the relationship between the asymptotic behavior of the recursive algorithms and the stability of the corresponding ODEs.
	
	To apply the ODE method, we first note that Algorithm \ref{alg1} can be rewritten in the following equivalent form:
	\begin{equation}\label{recuralg}
		\begin{aligned}
			&\beta_{k+1}={\beta}_{k}+\frac{1}{k}Q_1(x_k,\phi_k,y_{k+1}),\\
			&R_{k+1}=R_{k}+\frac{1}{k}Q_2(x_k,\phi_k,y_{k+1}),\\
		\end{aligned}
	\end{equation}
	with $Q_1(x_k,\phi_k,y_{k+1})=R_{k+1}^{-1}{\phi}_{k}\big(y_{k+1}\tanh\big(\frac{\beta_k^{\tau}\phi_ky_{k+1}}{\sigma^2}\big)-{\beta}_{k}^{\tau}{\phi}_{k}\big)$, $Q_2(x_k,\phi_k,y_{k+1})=\phi_k\phi_k^{\tau}-R_{k}$, and $x_{k}=\left[
	\begin{array}{cc}
		\beta_{k}^{\tau} & \text{vec}^{\tau}(R_k) \\
	\end{array}
	\right]^{\tau}$, where $\text{vec}(\cdot)$ denotes the operator by stacking the columns of a matrix on top of one another.
	Then $x_k$ evolves according to the following form:
	\begin{equation}\label{recf}	
		x_{k+1}=x_{k}+\frac{1}{k}Q(x_k,\phi_k,y_{k+1}),
	\end{equation}
	where 
	\begin{equation}\label{QQQ}
		Q(x_k,\phi_k,y_{k+1})=\left[
		\begin{aligned}
			&Q_1(x_k,\phi_k,y_{k+1})\\ &\text{vec}\left(Q_2(x_k,\phi_k,y_{k+1})\right)\\
		\end{aligned}\right].
	\end{equation} 
	In order to analyze (\ref{recf}), we introduce the following ODEs:
	\begin{subequations}\label{ode}
		\begin{align}
			&\frac{d}{dt}\beta(t)=R^{-1}(t)f(\beta(t)),\label{odeee}\\			
			&\frac{d}{dt}R(t)=G-R(t),\label{odee}
		\end{align}	
	\end{subequations}
	where $f(\beta(t))=\lim\limits_{k\to\infty}\mathbb{E}\big[\phi_{k}\big(y_{k+1}\tanh\big(\frac{\beta^{\tau}(t) \phi_{k}y_{k+1}}{\sigma^2}\big)-\beta^{\tau}(t) \phi_{k}\big)\big]$ and $G=\lim\limits_{k\to\infty}\mathbb{E}\left[\phi_k\phi_k^{\tau}\right]$.
	
	The main results of Ljung's ODE method can be restated in the following proposition, which plays an important role in our analysis:
	
\begin{proposition}\cite{ljung1977analysis}\label{ljungtheorem}
	%
	%
	%
	%
	%
	Let $D$ be an open and connected subset of $\mathbb{R}^d$ and $\bar D$ be a compact subset of $D$ such that the trajectory of (\ref{ode}) starting in $\bar D$ remains in $\bar D$ for $t>0$. 
	Let also that there is an invariant set $D_c\subset\bar D$ of (\ref{ode}) such that its attraction domain $D_A\supset\bar D$.
	Then  $x_k\to D_c$ as $k\to \infty$ almost surely, provided that the following conditions are satisfied:
	
	B1) The function $Q(x,\phi,y)$ defined in (\ref{recf}) is locally Lipschitz continuous for $x\in D$ with fixed $\phi$ and $y$, that is, for $x_i\in\mathcal{U}(x,\rho(x))$ with $\rho(x)>0$,  $$\|Q(x_1,\phi,y)-Q(x_2,\phi,y)\|<\mathcal{R}(x,\phi,y,\rho(x))\|x_1-x_2\|,$$
	where $x=\left[
	\begin{array}{cc}
		\beta^{\tau} & \text{vec}^{\tau}(R) \\
	\end{array}
	\right]^{\tau}$, and $\mathcal{U}(x,\rho(x))$ is the $\rho(x)$-neighborhood of $x$, i.e., $\mathcal{U}(x,\rho(x))=\{\bar x:\|x-\bar x\|<\rho(x)\}$.
	
	B2) $\frac{1}{n}\sum\limits_{k=1}^n\mathcal{R}(x,\phi_k,y_{k+1},\rho(x))$ converges to a finite limit for any $x\in D$ as $n\to\infty$.
	
	B3)  $\lim\limits_{k\to\infty}\mathbb{E} \left[Q(x,\phi_k,y_{k+1})\right]$ exists for $x\in D$ and 
	\begin{equation}\label{ergodic}
		\lim\limits_{n\to\infty}\frac{1}{n}\sum\limits_{k=1}^nQ(x,\phi_k,y_{k+1})=\lim\limits_{k\to \infty}\mathbb{E} \left[Q(x,\phi_k,y_{k+1})\right].
	\end{equation}
	
	B4) There exists a positive constant $L$ such that the following events happen $\text{i.o. with probability 1}$:
	$$x_k\in\bar D~\text{and}~\|\phi_k\|\leq L.$$		
\end{proposition}

In order to establish the convergence property of Algorithm \ref{alg1}, we will verify all the conditions in Proposition \ref{ljungtheorem}. 
For this purpose, we first provide a related lemma on the properties of the regressor, the hidden variable and the noise.		


\begin{lemma}\label{dk} 
	Under Assumptions \ref{asm2}-\ref{asm4}, $\{d_{k}\triangleq \left[
	\begin{array}{*{20}{c}}
		z_k & \phi_k^{\tau} &w_{k+1}		
	\end{array}
	\right]
	^{\tau}, k\geq 1\}$ is an asymptotically stationary and ergodic process with bounded fourth moment.
	In addition, any measurable function of $d_k$ is also an asymptotically stationary and ergodic stochastic process.
\end{lemma}

This lemma can be easily obtained by Assumptions \ref{asm2}-\ref{asm4} and the ergodicity theorem of the stationary process, and the proof details are provided in Appendix \ref{app1}.

\begin{lemma}\label{regular}
	Consider Algorithm \ref{alg1} subject to Assumptions \ref{asm2}-\ref{asm4}. Then Conditions B1)-B3) in Proposition \ref{ljungtheorem} are satisfied in the open area $D=\{x: R>0\}$ with $x=\left[
	\begin{array}{cc}
		\beta^{\tau} & \text{vec}^{\tau}(R)
	\end{array}
	\right]^{\tau}$.
\end{lemma}
\begin{proof}
	Let us first verify Condition B1).	
	For $x\in D$ and $x_i\in\mathcal{U}(x,\rho(x))$ ($i=1,2$), by (\ref{recf}), we have 
	$$
		\begin{aligned}
			\|Q_1(x_1,\phi,y)-Q_1(x_2,\phi,y)\|	\leq \mathcal{R}_1(x,\phi,y,\rho)\|x_1-x_2\|,
		\end{aligned}
	$$
	where
	\begin{equation}\label{R1}
		\begin{aligned}
			&\mathcal{R}_1=\sup\limits_{\mathcal{U}(x,\rho)}\bigg[\big\|\frac{\partial Q_1(x,\phi,y)}{\partial \beta}\big\|+\big\|\frac{\partial Q_1(x,\phi,y)}{\partial R}\big\|\bigg]\\
			\leq&\sup\limits_{\mathcal{U}(x,\rho)}\bigg[\frac{\|\phi\|^2}{\lambda_{min}(R)}\left(\frac{y^2}{\sigma^2}+1\right)+\frac{\|\phi\||y|+\|\phi\|^2\|\beta\|}{\lambda_{min}^2(R)}\bigg]
		\end{aligned}
	\end{equation}
	with $\rho(x)$ being sufficiently small such that any point $\bar x$ in the area $\mathcal{U}(x,\rho(x))$ has the property that $\lambda_{min}(\bar R)>0$, and the second inequality holds by $\frac{\partial R^{-1}}{\partial R}=-R^{-1}\otimes R^{-1}$ \cite{ahmed2008matrix},
	$$\begin{aligned}
		&\left\|\frac{\partial R^{-1}}{\partial R}\right\|=\left\|R^{-1}\otimes R^{-1}\right\|=\frac{1}{\lambda_{min}^2(R)},
	\end{aligned}
	$$
	where $\otimes$ is the Kronecker product for matrices.
	For $Q_2(x,\phi,y)$ in (\ref{recf}), we have $$
	\begin{aligned}
		\|Q_2(x_1,\phi,y)-Q_2(x_2,\phi,y)\|\leq \|x_1-x_2\|.
	\end{aligned}
	$$
	Thus for fixed $\phi$ and $y$, $Q(x,\phi,y)$ is locally Lipschitz continuous with respect to $x\in D$, and Condition B1) is satisfied.
	
	For the verification of Condition B2), we only need to prove that $\mathcal{R}_1$ satisfies B2) since $\mathcal{R}_2=1$.
	Moverover, since $\mathcal{R}_1$ is defined as the Lipschitz constant, we only need to verify its upper bound defined in (\ref{R1}) to satisfy B2), which is quite obvious because $\|\phi_k\|y_{k+1}$, $\|\phi_k\|^2y_{k+1}^2$ and $\|\phi_k\|^2$ are all asymptotically stationary and ergodic by Lemma \ref{dk} and model (\ref{model2}), and the supremum over $\rho(x)$ only concerns with $\beta$ and $R$ and is independent of the time instant $k$.
	
	The verification of Condition B3) is straightforward by Lemma \ref{dk} and model (\ref{model2}), and the details are omitted.		
	This completes the proof of Lemma \ref{regular}.
\end{proof}

Before formally proving Theorem \ref{odetheorem1}, we first provide some elementary lemmas with their proofs given in Appendix \ref{app1}.
\begin{lemma}\label{tanh1}
	For a random variable $y\sim p\mathcal{N}(a,\sigma^2)+(1-p)\mathcal{N}(-a,\sigma^2)$ with $p\in[0,1]$ and $a$ being a constant, we have $\mathbb{E}\left[y\tanh\left(\frac{ay}{\sigma^2}\right)\right]=a$.
\end{lemma}

\begin{lemma}\label{tanh2}\cite{daskalakis2017ten}
	For a random variable $y\sim\mathcal{N}(a,\sigma^2)$ with constants $a$ and $\hat a$ satisfying $a\hat a\geq0$, we have
	$$\begin{aligned}
		&\mathbb{E}\left[\frac{\hat ay}{\sigma^2}\tanh'\left(\frac{\hat ay}{\sigma^2}\right)\right]\geq0,\\
		&\mathbb{E}\left[\tanh\left(\frac{\hat ay}{\sigma^2}\right)\right]\geq 1-\exp\left(-\frac{|a|\min\left(|a|,|\hat a|\right)}{2\sigma^2}\right),
	\end{aligned}
	$$
	where $\tanh'(\cdot)$ is the derivative function of $\tanh(\cdot)$.
\end{lemma}

\begin{lemma}\label{le22}
	For $x>0$ and $c>0$, the function
	\begin{equation}\label{o1}
		F(c,x)=\int_{-\infty}^{\infty}f(c,x,w)\exp\left(-\frac{w^2}{2\sigma^2}\right)dw
	\end{equation}
	is increasing with respect to $x$, where $f(c,x,w)=(w+x)\tanh\left(\frac{c(w+x)}{\sigma^2}\right)+(w-x)\tanh\left(\frac{c(w-x)}{\sigma^2}\right)$.
\end{lemma}

\begin{lemma}\cite{huang1990estimation}\label{project}
	Suppose that $\Psi$ and $Y$ are any $n\times l$ and $n\times r$-dimensional matrices, respectively and $\Psi^{\tau}\Psi$ is invertible. Then
	$$Y^{\tau}\Psi (\Psi^{\tau}\Psi)^{-1}\Psi^{\tau} Y\leq Y^{\tau}Y.$$
\begin{lemma}\cite{hahn1967stability}\label{vvv}
	Consider the following ODE:
	$$
	\frac{dV(t)}{dt}=-V(t)+r(t),
	$$
	where $\lim\limits_{t\to\infty} r(t)=0$, then we have $\lim\limits_{t\to\infty} V(t)=0$.
\end{lemma}	
\end{lemma}
	\noindent\hspace{1em}{\textbf{\itshape Proof of Theorem \ref{odetheorem1}:}}
	We prove Theorem \ref{odetheorem1} by verifying all the conditions required in Proposition \ref{ljungtheorem}.
	By Lemma \ref{regular}, it only remains to verify the required properties of the trajectory of the ODEs (\ref{ode}) (see Steps 1-3) as well as Condition B4) (see Step 4).
	Without loss of generality, we focus on the analysis for the case of $\beta^*\ne0$ and then give some additional explanations for the case of $\beta^*=0$.
	
	We now investigate the properties of (\ref{ode}). 	
	Specifically, we prove the following assertion:
	\begin{equation}\label{fact}
		\begin{aligned}
			&\text{the ODEs (\ref{ode}) has the invariant set}~D_{c}=\{x:\beta=\\
			&\{\beta^*,-\beta^*,0\},R=G\}\ \hbox{with the domain of attraction}~\\
			&D_{A}=\{x:\|R-G\|<\varepsilon\}.
		\end{aligned}		
	\end{equation}
	where $x=\left[
	\begin{array}{cc}
		\beta^{\tau} & \text{vec}^{\tau}(R)
	\end{array}
	\right]^{\tau}$ and $\varepsilon$ is a positive constant to be determined later. 		
	We prove the assertion (\ref{fact}) by establishing the stability properties of (\ref{ode}) on three subsets
	$D_{A,1}=\{x:\beta^{\tau}\beta^*>0, \|R-G\|<\varepsilon\},$
	$D_{A,2}=\{x:\beta^{\tau}\beta^*<0, \|R-G\|<\varepsilon\}$ and
	$D_{A,3}=\{x:\beta^{\tau}\beta^*=0, \|R-G\|<\varepsilon\}$ of the domain of attraction $D_A$.	
	For the case of $x(0)\in D_{A,1}$, we show that $x(t)\in D_{A,1}$ for $t\geq0$ in Steps 1-2 below and then establish the corresponding stability result in Step 3 below.  
	
	From (\ref{odee}), we have
	\begin{equation}\label{R11}
		R(t)=G+e^{-t}(R(0)-G).
	\end{equation}
	By (\ref{ode}) and the fact that the asymptotically stationary p.d.f of $\phi_k$ has the rotation invariant property, we have $G=c_0I$ with $c_0$ being a positive constant \cite{fang2018symmetric}.
	From (\ref{R11}), it follows that $\|R(t)-G\|\leq\|R(0)-G\|$. 
	So it suffices to show $\beta^{\tau}(t)\beta^*>0$ when proving $x(t)\in D_{A,1}$ for $t\geq0$.		
	For this, we derive the lower and upper bounds of $\beta(t)$ in (\ref{odeee}) in Step 1 below. 
	
	\textbf{Step 1: Boundedness of $\|\beta(t)\|$ generated by (\ref{odeee}).}	
	
	For the convenience of analysis, we choose a set of standard orthogonal basis $\{v_1(t),\cdots,v_{d}(t)\}$, where $v_1(t)=\frac{\beta(t)}{\|\beta(t)\|}$ and $v_2(t)$ belongs to span$\{\beta(t),\beta^*\}$. 
	Then $\beta(t)$, $\beta^*$ and $f(\beta(t))$ defined in (\ref{ode}) can be written as follows:
	\begin{equation}\label{betafenjie}
		\begin{aligned}
			&\beta(t)=\sum\nolimits_{i=1}^db_i(t)v_i(t),~\beta^*=\sum\nolimits_{i=1}^db_i^*(t)v_i(t),\\
			&f(\beta(t))=\sum\nolimits_{i=1}^dh_i(\beta(t))v_i(t),
		\end{aligned}		
	\end{equation}
	where $b_{1}(t)=\|\beta(t)\|$, $b_i(t)\equiv0 (i>1)$, $b_1^*(t)=\beta^{*\tau}v_1(t)$, $b_2^*(t)=\beta^{*\tau}v_2(t)$, $b_i^*(t)\equiv0 (i>2)$ and $h_i(\beta(t))=v_i^{\tau}(t)f(\beta(t))$. Note that there may exist a time instant $t_0$  such that $v_2(t)=v_1(t)$ for $t\geq t_0$. 
	For such a case, the orthogonal basis defined above degenerates, but the analysis in this part still holds. Moreover, we give some illustrations on properties of $f(\beta(t))$. By (\ref{ode}) and Assumptions \ref{asm3}-\ref{asm4}, we have 
	\begin{equation}\label{fbeta}
		f(\beta(t))=\mathbb{E}\big[\phi\big(y\tanh\big(\frac{\beta^{\tau}(t)\phi y}{\sigma^2}\big)-\beta^{\tau}(t) \phi\big)\big],
	\end{equation}
	where $\phi$ is a random vector with p.d.f $\bar g\in\mathcal{G}$ being the asymptotic stationary p.d.f of $\phi_k$ and the random variable $y$ given $\phi$ obeys the distribution $\mathcal{N}(\beta^{*\tau}\phi,\sigma^2)$ by Assumption \ref{asm3} and Lemma \ref{tanh1}.	 
	Denote $a_i(t)=v_i^{\tau}(t)\phi$, $i\in[d]$, we have
	\begin{equation}\label{phi}
		\phi=\sum\nolimits_{i=1}^da_i(t)v_i(t).
	\end{equation}
	From (\ref{phi}) and the rotation-invariant property of the p.d.f of $\phi$ by Assumption \ref{asm4}, the p.d.f of $a(t)=[a_1(t),\cdots,a_d(t)]^{\tau}$ equals $\bar g(a(t))$ which is an even function in $a_i(t)$ and also the marginal p.d.f of $a_i(t)$ is an even function in $a_i(t)$ for $i\in[d]$.
	So we have $\mathbb{E}[a_1(t)a_i(t)]=0, i\in[d]\backslash\{1\}$ and $\mathbb{E}[a_1^{3}(t)a_2(t)]=0.$
	Moreover, by the definition of $G$ in (\ref{ode}) and $G=c_0I$ in (\ref{R11}), we have $\mathbb{E}[\phi\phi^{\tau}]=c_0I$, thus $\mathbb{E}[a_i^2(t)]=c_0$ for $i\in[d]$. Besides, by the rotation-invariant property of the p.d.f of $\phi$, we also have $\mathbb{E}[a_i^{4}(t)]=c_1, i\in[d]$ with a positive constant $c_1$. These properties will be used in the following analysis without citations. 
	
	We now prove that $b_1(t)=\|\beta(t)\|$ has a positive lower bound for $t\geq0$.
	For this purpose, it suffices to prove that there exists a constant $b_l>0$ such that 
	\begin{equation}\label{p1}
		\frac{db_1(t)}{dt}>0, \ \hbox{if}~ 0<b_1(t)<b_l.
	\end{equation}   
	By the fact $v_i^{\tau}(t)v_i(t)\equiv1$ for $i\in[d]$, we have
	\begin{equation}\label{basis}
		v_{i}^{\tau}(t)\frac{dv_i(t)}{dt}\equiv0.
	\end{equation}
	Thus from (\ref{odeee}), (\ref{betafenjie}) and (\ref{basis}), we have
	\begin{equation}\label{db1}
		\begin{aligned}
			&\frac{db_1(t)}{dt}=\frac{d[\beta^{\tau}(t)v_1(t)]}{dt}\\
			=&v_1^{\tau}(t)\frac{d\beta(t)}{dt}+b_1(t)v_1^{\tau}(t)\frac{dv_1(t)}{dt}=v_1^{\tau}(t)R^{-1}(t)f(\beta(t))\\
			=&c_0^{-1}h_1(\beta(t))+v_1^{\tau}(t)(R^{-1}(t)-c_0^{-1}I)f(\beta(t))\buildrel \Delta \over=S(\beta(t)).
		\end{aligned}		
	\end{equation}	
	By (\ref{betafenjie}) and (\ref{phi}), we have $\beta^{\tau}(t)\phi=a_1(t)b_1(t)$ and $\beta^{*\tau}\phi=a_1(t)b_1^{*}(t)+a_2(t)b_2^{*}(t)$, thus by (\ref{fbeta}), we have for $i\in[d]$,
	\begin{equation}\label{h11}
		h_i(\beta(t))=\mathbb{E}\big[a_i(t)\big(y\tanh\big(\frac{a_1(t)b_1(t)y}{\sigma^2}\big)-a_1(t)b_1(t)\big)\big],
	\end{equation}
	where $y$ obeys the distribution $\mathcal{N}(a_1(t)b_1^{*}(t)+a_2(t)b_2^{*}(t),\sigma^2)$ given $a(t)$.
	Then by $\tanh(0)=0$, we have $h_i(\beta(t))=0$ and $f(\beta(t))=0$ for $b_1(t)=0$, thus $S(\beta(t))=0$ for $b_1(t)=0$.
	Hence, by (\ref{db1}) and the mean-value theorem, we obtain
	\begin{equation}\label{b1'}
			\left.S(\beta(t))=(b_1(t)-0)\frac{dS(\beta(t))}{db_1(t)}\right|_{b_1(t)=\zeta(t)},		
	\end{equation}
	where $\zeta(t)\in[0,b_1(t)]$.
	To prove (\ref{p1}), i.e., $S(\beta(t))>0$ if $0<b_1(t)<b_l$, we proceed to show
	\begin{equation}\label{dbb}
		\frac{dS(\beta(t))}{db_1(t)}\bigg|_{b_1(t)=0}>0.
	\end{equation}	
	By (\ref{betafenjie}) and $f(\beta(t))=0$ when $b_1(t)=0$, we have
	\begin{equation}\label{dhh}
		\begin{aligned}
			&\frac{dS(\beta(t))}{db_1(t)}\bigg|_{b_1(t)=0}=c_0^{-1}\frac{d h_1(\beta(t))}{d b_1(t)}\bigg|_{b_1(t)=0}+v_1^{\tau}(t)\\
			&~~(R^{-1}(t)-c_0^{-1}I)\bigg(\sum\limits_{i=1}^d\frac{v_i(t)d h_i(\beta(t))}{d b_1(t)}\bigg)\bigg|_{b_1(t)=0}.
		\end{aligned}	
	\end{equation}	
	To analyze (\ref{dhh}), we consider $\frac{d h_i(\beta(t))}{d b_1(t)}\big|_{b_1(t)=0}, i\in [d]$ term by term.  Firstly, by Assumptions \ref{asm2}-\ref{asm4}, the definitions of $y$ in (\ref{fbeta}) and $a_1(t)$ in (\ref{phi}), we have $\mathbb{E}[a_1^2(t)]<\infty$, $\mathbb{E}[a_1^2(t)y^2]<\infty$.
	Then by (\ref{h11}), $\tanh'(0)=1$, $\|b_1^*(t)\|^2+\|b_2^*(t)\|^2=\|\beta^*\|^2$ and the properties of $a(t)$, we obtain
	\begin{equation}\label{dh1}
		\begin{aligned}
			&\frac{d h_1(\beta(t))}{d b_1(t)}\bigg|_{b_1(t)=0}
			=\mathbb{E}\left[\frac{a_1^2(t)y^2}{\sigma^2}-a_1^2(t)\right]\\
			=&\mathbb{E}\left[a_1^2(t)\frac{\sigma^2+(a_1(t)b_1^*(t)+a_2(t)b_2^*(t))^2}{\sigma^2}-a_1^2(t)\right]\\
			=&\mathbb{E}\left[\frac{a_1^4(t)b_1^{*2}(t)+a_1^2(t)a_2^2(t)b_2^{*2}(t)}{\sigma^2}\right]>0.
		\end{aligned}
	\end{equation}
	Secondly, by (\ref{h11}), we have
	\begin{equation}\label{dh2}
		\begin{aligned}
			&h_2(\beta(t))=\mathbb{E}\big[a_2(t)y\tanh\big(\frac{a_1(t)b_1(t)y}{\sigma^2}\big)\big].
		\end{aligned}
	\end{equation}
	Similar to (\ref{dh1}), with simple calculations, it follows that
	\begin{equation}\label{dh22}
		\frac{d h_2(\beta(t))}{d b_1(t)}\bigg|_{b_1(t)=0}=2\mathbb{E}\left[\frac{a_1^2(t)a_2^2(t)}{\sigma^2}\right]b_1^*(t)b_2^{*}(t).
	\end{equation}	
	Thirdly, by (\ref{h11}) and the fact that the marginal p.d.f of $a_i(t)$ is even, it is not difficult to obtain that for $i>2$,
	\begin{equation}\label{dh3}
		h_i(\beta(t))\equiv0,\ \hbox{and}~ \frac{dh_i(\beta(t))}{db_1(t)}\bigg|_{b_1(t)=0}=0.
	\end{equation}	
	Choose $\varepsilon=\frac{1}{4}c_0$ in $D_{A,1}$, we have
	$-\frac{1}{5c_0}I<R^{-1}(t)-c_0^{-1}I<\frac{1}{3c_0}I$.
	Thus by (\ref{dhh}), (\ref{dh1}), (\ref{dh22}), (\ref{dh3}) and the properties of $a(t)$, we obtain	
	$$
	\begin{aligned}
		&\frac{dS(\beta(t))}{db_1(t)}\bigg|_{b_1(t)=0}
		>\frac{4\left[\mathbb{E}[a_1^4(t)]b_1^{*2}(t)+\mathbb{E}[a_1^2(t)a_2^2(t)]b_2^{*2}(t)\right]}{5c_0\sigma^2}\\
		&-\frac{2\mathbb{E}[a_1^2(t)a_2^2(t)]b_1^*(t)b_2^{*}(t)}{3c_0\sigma^2}\geq\frac{\|\beta^*\|^2\mathbb{E}\left[a_1^2(t)a_2^2(t)\right]}{3c_0\sigma^2}>0.
	\end{aligned}	
	$$ 
	Thus (\ref{dbb}) is obtained. From the continuity of $\frac{dS(\beta(t))}{db_1(t)}$ at $b_1(t)=0$, it follows that there exists a positive constant $b_l$ such that $\frac{dS(\beta(t))}{db_1(t)}\big|_{b_1(t)=\zeta(t)}>0$ if $0<\zeta(t)\leq b_1(t)\leq b_l$.
	Then by (\ref{b1'}), we have $S(\beta(t))>0$ if $0<b_1(t)\leq b_l$, thus (\ref{p1}) holds.
	Therefore, we can obtain that $b_1(t)$ has the following positive lower bound:
	\begin{equation}\label{bound1}
		\begin{aligned}
			&b_1(t)\geq\underline b=\min\left\{b_l,b_1(0)\right\}.\\
		\end{aligned}		
	\end{equation}	
	
	We now derive an upper bound of $b_1(t)=\|\beta(t)\|$. 	
	By (\ref{h11}), $|\tanh(\cdot)|\leq1$ and Assumptions \ref{asm3}-\ref{asm4}, we have
	\begin{equation}\label{h1}
			h_1(\beta(t))
			\leq\mathbb{E}[|a_1(t)y|]-b_1(t)\mathbb{E}[a_1^2(t)]=p_0-c_0b_1(t),
	\end{equation}
	where $p_0=\mathbb{E}\left[|a_1(t)y|\right]<\infty$. 
	By (\ref{dh2}), (\ref{dh3}) and (\ref{h1}), we have
	\begin{equation}\label{h22}
		\begin{aligned}
			\|f(\beta(t))\|\leq\|h_1(\beta(t))+h_2(\beta(t))\|\leq p_1+c_0b_1(t),
		\end{aligned}
	\end{equation}
	where $p_1=p_0+\mathbb{E}[|a_2(t)y|]<\infty$.
	Thus by (\ref{db1}), we have
	\begin{equation}
		\frac{db_1(t)}{dt}\leq c_0^{-1}(p_0-c_0b_1(t))+\frac{1}{3}c_0^{-1}(p_1+c_0b_1(t))\leq0,
	\end{equation}	
	 if $b_1(t)\geq\frac{3p_0+p_1}{2c_0}$. Therefore, $b_1(t)$ has the following upper bound:
	\begin{equation}\label{bound2}
		\begin{aligned}
			b_1(t)\leq\bar b=\max\left\{(3p_0+p_1)/(2c_0),b_1(0)\right\}.
		\end{aligned}		
	\end{equation}	
	
	\textbf{Step 2: Proof of $\beta^{\tau}(t)\beta^*>0$ for all $t\geq0$.}	
	
	For this purpose, by $\beta^{\tau}(t)\beta^*=b_1(t)b_1^*(t)$ and $b_1(t)>0$ derived in Step 1, we only need to prove that $b_1^*(t)>0$ for $t\geq0$.	
	For this purpose, by the fact that $b_1^*(0)>0$ for $x(0)\in D_{A,1}$, it suffices to show that
	\begin{equation}\label{p11}
		\frac{db_1^*(t)}{dt}\geq 0, \ \hbox{if} \ b_1^*(t)>0.
	\end{equation}
	By (\ref{betafenjie}), we have $b_1^{*2}(t)+b_2^{*2}(t)\equiv\|\beta^*\|^2$ and then
	\begin{equation}\label{derivitive}
		b_1^*(t)\frac{db_1^*(t)}{dt}=-b_2^*(t)\frac{db_2^*(t)}{dt}.
	\end{equation}
	In order to prove (\ref{p11}), we first analyze the properties of $\frac{db_2^{*}(t)}{dt}$ for $0\leq| b_2^{*}(t)|<\|\beta^{*}\|$. 
	By $b_2(t)\equiv0$ in (\ref{betafenjie}), it follows that $\frac{db_2(t)}{dt}\equiv0$, then by (\ref{ode}) and (\ref{betafenjie}), we have
	\begin{equation}\label{b1t}
		\begin{aligned}
			&~\frac{db_2(t)}{dt}=\frac{d[\beta^{\tau}(t)v_2(t)]}{dt}=\beta^{\tau}(t)\frac{dv_2(t)}{dt}+v_2^{\tau}(t)\frac{d\beta(t)}{dt}\\
			&=b_1(t)v_1^{\tau}(t)\frac{dv_2(t)}{dt}+v_2^{\tau}(t)R^{-1}(t)f(\beta(t))\equiv0.
		\end{aligned}
	\end{equation}
	Thus by (\ref{basis}) and (\ref{b1t}), we have
	\begin{equation}\label{b2*t}
		\begin{aligned}
			&\frac{db_2^*(t)}{dt}=\beta^{*\tau}\frac{dv_2(t)}{dt}=b_1^*(t)v_1^{\tau}(t)\frac{dv_2(t)}{dt}\\
			=&-\frac{b_1^*(t)}{b_1(t)}v_2^{\tau}(t)R^{-1}(t)f(\beta(t))=-\frac{b_1^*(t)}{b_1(t)}c_0^{-1}h_2(\beta(t))\\
			&~~-\frac{b_1^*(t)}{b_1(t)}v_2^{\tau}(t)(R^{-1}(t)-c_0^{-1}I)f(\beta(t)).
		\end{aligned}		
	\end{equation}	
	We analyze the right-hand-side (RHS) of (\ref{b2*t}) term by term.
	For the first term, let us denote the marginal p.d.f of $(a_1(t),a_2(t))$ as follows: 
	\begin{equation}\label{tildeg}
		\begin{aligned}
			&\tilde g(a_1(t),a_2(t))\\
			=&\int_{\mathbb{R}}\cdots\int_{\mathbb{R}}\bar g(a_1(t),\cdots,a_d(t))da_3(t)\cdots da_d(t),
		\end{aligned}		
	\end{equation}
	where $\bar g$ is the p.d.f of the random vector $\phi$ defined in (\ref{fbeta}). 
	Since $\tanh(z)$ is odd and $z\tanh(z)$ is even in $z$, by (\ref{dh2}), we have 
	\begin{equation}\label{h2}
		\begin{aligned}
			&h_2(\beta(t))=\mathbb{E}\bigg[a_2(t)\big[a_1(t)b_1^*(t)+a_2(t)b_2^*(t)+w\big]\\
			&\tanh\bigg(\frac{a_1(t)b_1(t)\big[a_1(t)b_1^*(t)+a_2(t)b_2^*(t)+w\big]}{\sigma^2}\bigg)\bigg]\\
			=&\frac{1}{\sqrt{2\pi}\sigma}\int_{\mathbb{R}}\int_{\mathbb{R}}\int_{\mathbb{R}}a_2(t)\big[a_1(t)b_1^*(t)+a_2(t)b_2^*(t)+w\big]\\
			&\tanh\bigg(\frac{a_1(t)b_1(t)\big[a_1(t)b_1^*(t)+a_2(t)b_2^*(t)+w\big]}{\sigma^2}\bigg)\\
			&~~~~~~~\tilde g(a_1(t),a_2(t)))\exp\left(-\frac{w^2}{2\sigma^2}\right)dwda_2(t)da_1(t)\\
			=&\frac{1}{\sqrt{2\pi}\sigma}\int_{a_1(t)>0}\int_{a_2(t)>0}\big\{
			I(a_1(t),a_2(t))\\
			&~~~~+I(-a_1(t),a_2(t))+I(-a_1(t),-a_2(t))\\
			&~~~~+I(a_1(t),-a_2(t))\big\}\tilde g(a_1(t),a_2(t))da_2(t)da_1(t)\\
			=&\frac{1}{\sqrt{2\pi}\sigma}\int_{a_1(t)>0}\int_{a_2(t)>0}a_2(t)\big\{
			F\big(c(t),|a_1(t)b_1^*(t)\\
			&+a_2(t)b_2^*(t)|\big)-F\big(c(t),|a_1(t)b_1^*(t)-a_2(t)b_2^*(t)|\big)\big\}\\
			&~~~~~~~~~~~~~~~~~~~~~~~~~~~\tilde g(a_1(t),a_2(t))da_2(t)da_1(t),
		\end{aligned}
	\end{equation}
	where 
	$I(a_1(t),a_2(t))=a_2(t)\int_{\mathbb{R}^{1}}[a_1(t)b_1^*(t)+a_2(t)b_2^*(t)+w]\tanh\bigg(\frac{a_1(t)b_1(t)[a_1(t)b_1^{*}(t)+a_2(t)b_2^*(t)+w]}{\sigma^2}\bigg)\exp\left(-\frac{w^2}{2\sigma^2}\right)dw$, the function $F(c(t),x(t))$ is defined in Lemma \ref{le22} with $c(t)=a_1(t)b_1(t)>0$ (since $b_1(t)>0$ by (\ref{bound1})).
	Furthermore, for positive $a_1(t)$, $a_2(t)$ and $b_1^*(t)$, it follows that $|a_1(t)b_1^*(t)+a_2(t)b_2^*(t)|-|a_1(t)b_1^*(t)-a_2(t)b_2^*(t)|$ has the same sign as $b_2^*(t)$. 
	Hence, by Lemma \ref{le22}, $F\big(c(t),|a_1(t)b_1^*(t)+a_2(t)b_2^*(t)|\big)-F\big(c(t),|a_1(t)b_1^*(t)-a_2(t)b_2^*(t)|\big)$ has the same sign as $b_2^*(t)$.
	Therefore, by (\ref{h2}), we have
	\begin{equation}\label{second}
		\begin{aligned}
			&h_2(\beta(t))\geq 0~\text{if}~ 0\leq b_2^*(t)<\|\beta^*\|,\\
			&h_2(\beta(t))\leq0~\text{if}~-\|\beta^*\|<b_2^*(t)\leq0, 
		\end{aligned}		
	\end{equation}	
	where the equality holds if and only if $b_2^*(t)=0$.	
	Choose $\varepsilon=\frac{c_0^{-1}|h_2(\beta(0))|}{p_1+c_0\bar b}$, by (\ref{h22}), (\ref{bound2}) and (\ref{b2*t}), we obtain that $\frac{db_2^*(t)}{dt}$ has an opposite sign with $b_2^*(t)$ at $t=0$.
	Besides, from (\ref{R11}), it follows that $\|R^{-1}(t)-c_0^{-1}I\|=O(e^{-t})$. Then by (\ref{b2*t}), (\ref{h2}), (\ref{second}), we can derive that    
	\begin{equation}\label{db2}
		\begin{aligned}
			&\frac{db_2^*(t)}{dt}\leq 0~\text{if}~0\leq b_2^*(t)<\|\beta^*\|,\\ &\frac{db_2^*(t)}{dt}\geq0~\text{if}~-\|\beta^*\|<b_2^*(t)\leq0,
		\end{aligned}		
	\end{equation}
	where the equality holds if and only if $b_2^*(t)=0$.
	The detailed proof of (\ref{db2}) is provided in Appendix \ref{app1}. By taking the Lyapunov function as $b_2^{*2}(t)$ and using the Lasalle invariance principle, we obtain
	\begin{equation}\label{cvegb}
		\lim\limits_{t\to\infty}b_2^*(t)=0, \ \hbox{and }|b_2^*(t)|\leq| b_2^*(0)|.
	\end{equation}	
	Therefore, by (\ref{derivitive}) and (\ref{db2}), (\ref{p11}) is proved. Moreover, by (\ref{cvegb}) and $b_1^*(t)^2+b_2^*(t)^2\equiv\|\beta^*\|^2$, we have 
	\begin{equation}\label{b*}
		b_1^*(0)\leq b_1^*(t)\leq\|\beta^*\|.
	\end{equation}
	
	Combining all the above analysis, let 
	\begin{equation}\label{varep}
		\varepsilon\triangleq\min\left\{\frac{1}{4}c_0,\frac{c_0^{-1}|h_2(\beta(0))|}{dp_0+c_0\bar b}\right\},
	\end{equation}
	then by (\ref{bound1}) and (\ref{b*}), we obtain $\beta^{\tau}(t)\beta^*=b_1(t)b_1^*(t)>\underline bb_1^*(0)>0$ for all $t\geq 0$.
	
	\textbf{Step 3: Analysis of the Lyapunov function.}	
	
	We first establish the stability properties of $\beta(t)$ in (\ref{odeee}) for the case of $x(0)\in D_{A,1}$ using the Lyapunov method.
	 
	Consider the following Lyapunov function: $$V(\beta(t))=\frac{1}{2}\tilde\beta^{\tau}(t)R(t)\tilde\beta(t),$$ where $\tilde\beta(t)=\beta(t)-\beta^*$.
	We have the following derivative of $V$ along the trajectories (\ref{ode}):
	\begin{equation}\label{rdv}
		\frac{dV(\beta(t),R(t))}{dt}=\tilde\beta^{\tau}(t)f(\beta(t))+\frac{1}{2}\tilde\beta^{\tau}(t)\left(G-R(t)\right)\tilde\beta(t).
	\end{equation}
	For the first term on the RHS of (\ref{rdv}), by (\ref{betafenjie}), we have
	\begin{equation}\label{dv1}
		\tilde\beta^{\tau}(t)f(\beta(t))=\tilde b_1(t)h_1(\beta(t))-b_2^*(t)h_2(\beta(t)),	
	\end{equation}		
	where $\tilde b_1(t)=b_1(t)-b_1^*(t)$. With simple calculations, we have by (\ref{h11}),
	\begin{equation}\label{v1}
		\tilde b_1(t)h_1(\beta(t))=\tilde b_1(t)(L_1(\beta(t))+L_2(\beta(t))),
	\end{equation}	
	where 
	$$\begin{aligned}
		&L_1(\beta(t))=\mathbb{E}\bigg[a_1(t)\big[a_1(t)b_1^*(t)+a_2(t)b_2^*(t)+w\big]\times\\
		&~\tanh\left(\frac{a_1(t)b_{1}(t)[a_1(t)b_1^*(t)+a_2(t)b_2^*(t)+w]}{\sigma^2}\right)-a_1(t)\times\\
		&~~[a_1(t)b_1^*(t)+w]\tanh\left(\frac{a_1(t)b_{1}(t)[a_1(t)b_1^*(t)+w]}{\sigma^2}\right)\bigg],
	\end{aligned}$$
	and $$
	\begin{aligned}
		&L_2(\beta(t))=\mathbb{E}\bigg[a_1(t)\big[a_1(t)b_1^*(t)+w\big]\times\\
		&\tanh\bigg(\frac{a_1(t)b_{1}(t)[a_1(t)b_1^*(t)+w]}{\sigma^2}\bigg)-a_1^2(t)b_1(t)\bigg],
	\end{aligned}$$
	with $w\sim\mathcal{N}(0,\sigma^2)$ given $a_1(t)$ and $a_2(t)$.	
	By mean-value theorem, on the one hand, we have
	\begin{equation}\label{dvf2}
		\begin{aligned}
			&\tilde b_1(t)L_1(\beta(t))=\tilde b_1(t)b_2^*(t)\mathbb{E}\big[a_1(t)a_2(t)\big\{\tanh\left(l_1(t)\right)\\
			&~~~~~+l_1(t)\tanh'\left(l_1(t)\right)\big\}\big]
			\leq1.2c_0|\tilde b_1(t)||b_2^*(t)|,
		\end{aligned}
	\end{equation}
	where  $l_1(t)=\frac{a_1(t)b_1(t)(a_1(t)b_1(t)+a_2(t)\zeta_1(t)+w)}{\sigma^2}$, $\zeta_1(t)$ is between 0 and $b_2^*(t)$, and the last inequality holds by $|\tanh(z)+z\tanh'(z)|\leq1.2$ and Schwarz inequality.	
	On the other hand, by Lemma \ref{tanh1} and Assumption \ref{asm3}, we have $a_1(t)b_1(t)=\mathbb{E}\big[[a_1(t)b_1(t)+w]\tanh\big(\frac{a_1(t)b_{1}(t)[a_1(t)b_1(t)+w]}{\sigma^2}\big)\big|a_1(t)\big],$
	then it follows that
	$$
	\begin{aligned}
		L_2&(\beta(t))=\mathbb{E}\bigg[a_1(t)\big[a_1(t)b_1^*(t)+w\big]\times\\
		&\tanh\left(\frac{a_1(t)b_{1}(t)[a_1(t)b_1^*(t)+w]}{\sigma^2}\right)-a_1(t)\times\\
		&\big[a_1(t)b_1(t)+w\big]\tanh\left(\frac{a_1(t)b_1(t)[a_1(t)b_1(t)+w]}{\sigma^2}\right)\bigg]\\
		=&-\tilde b_1(t)\mathbb{E}\big[a_1^2(t)\big\{\tanh\left(l_2(t)\right)+l_2(t)\tanh'\left(l_2(t)\right)\big\}\big],\\	
	\end{aligned}
	$$
	where  $l_2(t)=\frac{a_1(t)b_1(t)[a_1(t)\zeta_2(t)+w]}{\sigma^2}$, $\zeta_2(t)$ is between $b_1^*(t)$ and $b_1(t)$. By (\ref{bound1}) and (\ref{b*}), we have $\min{(\zeta_2^2(t),\zeta_2(t)b_1(t))}\geq\min{(b_1^{*2}(0),\underline{b}^2)}\triangleq\underline c>0$. By Lemma \ref{tanh2}, we can derive that
	\begin{equation}\label{dvf4}
		\begin{aligned}
			\tilde b_1(t)L_2(\beta(t))\leq-C\tilde b_1^2(t),
		\end{aligned}
	\end{equation}
	where $C=\mathbb{E}\big[a_1^2(t)\big(1-\exp\big(-\frac{\underline ca_1^2(t)}{2\sigma^2}\big)\big)\big]$
	is a positive constant by Assumption \ref{asm4}.	
	Then by (\ref{rdv})-(\ref{dvf4}), $\tilde\beta^2(t)=\tilde b_1^2(t)+b_2^{*2}(t)$ from (\ref{betafenjie}), and $G=c_0I$, we have
	\begin{equation}\label{dvbeta}
		\begin{aligned}
			&\frac{d}{dt}V(\beta(t),R(t))\leq -C\tilde b_1^2(t)+r_1(t)\\
			\leq&-Cc_0^{-1}V(\beta(t),R(t))+r_2(t),	
		\end{aligned}
	\end{equation}
	where $r_1(t)=1.2c_0|\tilde b_1(t)||b_2^*(t)|-b_2^*(t)h_2(\beta(t))+\frac{1}{2}\tilde\beta^{\tau}(t)\left(G^{-1}-R(t)\right)\tilde\beta(t)$ and $r_2(t)=r_1(t)+C b_2^{*2}(t)+Cc_0^{-1}\tilde\beta^{\tau}(t)\left(G^{-1}-R(t)\right)\tilde\beta(t)$.
	By (\ref{R11}) and (\ref{cvegb}), we have $\lim\limits_{t\to\infty}r_2(t)=0$.
	Hence, by Lemma \ref{vvv} and the comparison principle \cite{khalil2002nonlinear}, we obtain $\lim\limits_{t\to\infty}V(\beta(t),R(t))=0$, then from the positive-definiteness property of $R(t)$, we have $\lim\limits_{t\to\infty}\beta(t)=\beta^*$. Therefore, we obtain that $D_{c,1}=\{x:\beta=\beta^*,R=G\}$ is the invariant set with domain of attraction $D_{A,1}$.
	
	We now provide the additional analysis for the other two cases of $x(0)\in D_{A,2}$ and $x(0)\in D_{A,3}$.
	
	For the case of $x(0)\in D_{A,2}$, similar to the analysis for the case $x(0)\in D_{A,1}$, we can obtain that $D_{c,2}=\{x:\beta=-\beta^*,R=G\}$ is the invariant set with the domain of attraction $D_{A,2}$ by choosing the Lyapunov function as  $\frac{1}{2}\tilde\beta^{\tau}(t)R(t)\tilde\beta(t)$ with $\tilde\beta(t)=\beta(t)+\beta^*$.
	
	For the case of $x(0)\in D_{A,3}$, we have $b_1^*(0)=0$ and $b_2^*(0)=\|\beta^*\|$. Then by (\ref{b2*t}), we can obtain $\frac{db_2^*(t)}{dt}\equiv0$ if $b_1^*(t)=0$ and $b_2^*(t)=\|\beta^*\|$ and thus $b_2^*(t)\equiv\|\beta^*\|$ and $b_1^*(t)\equiv0$ for all $t\geq0$. 
	Furthermore, by (\ref{fbeta}), it follows that
	\begin{equation}\label{hh1}
		\begin{aligned}
			&h_1(\beta(t))=\mathbb{E}\bigg[a_1(t)\big[a_2(t)b_2^*(t)+w\big]\times\\
			&~\tanh\bigg(\frac{a_1(t) b_{1}(t)[a_2(t)b_2^*(t)+w]}{\sigma^2}\bigg)\bigg]-\mathbb{E}[a_1^2(t)b_1(t)]\\
			&=-\mathbb{E}[a_1^2(t)b_1(t)]=-c_0b_1(t),
		\end{aligned}
	\end{equation}
	where the second equality holds by following a similar way as (\ref{h2}). Moreover, by (\ref{h2}), we have $h_2(\beta(t))=0$, and (\ref{dh3}) still holds.	
	Thus by (\ref{R11}), (\ref{db1}), (\ref{h1}) and (\ref{hh1}), we have
	\begin{equation}\label{dv}
		\frac{db_1(t)}{dt}\leq -b_1(t)+p_0c_0^{-1}e^{-t}.
	\end{equation}
	By Lemma \ref{vvv}, it is clear that $\lim\limits_{t\to\infty}b_1(t)=0$ and thus we have that $D_{c,3}=\{x:\beta=0,R=G\}$ is the invariant set with domain of attraction $D_{A,3}$.
	
	Combining all the above analysis, we see that the assertion (\ref{fact}) is true. By Proposition \ref{ljungtheorem}, the remaining proof concerns about the compact set $\bar D$. We give a specific expression of $\bar D$ as follows:
	\begin{equation}\label{bard1}
		\bar D=\{x:\|\beta\|\leq\max\{m_0,\bar b\} ,\|R-G\|\leq\varepsilon_1\},
	\end{equation} 
	where $m_0=\sqrt{2\|\beta^*\|^2+2c_0^{-1}\sigma^2}$, $0<\varepsilon_1<\varepsilon$, $\bar b$ and $\varepsilon$ are defined in (\ref{bound2}) and (\ref{varep}), respectively. 
	From $\|\beta(t)\|=b_1(t)$, (\ref{bound2}) and (\ref{R11}), it is clear that the trajectory of (\ref{ode}) that starts in $\bar D$ remains in $\bar D$ for $t>0$. 
	
	\textbf{Step 4: Convergence of the sequence $\{\beta_k\}$}
	
	For the remaining proof, we need to verify Condition B4) in Proposition \ref{ljungtheorem}.
	By Lemma \ref{dk}, it follows that $\{\phi_k\}$ is bounded i.o. with probability 1.
	We only need to prove that the event $\{x_k\in\bar D, k\geq0\}$ happens i.o. with probability 1.
	
	We now analyze the properties of $R_{k+1}$ and $\beta_k$, respectively.
	
	Firstly, by (\ref{recuralg}) and Assumption \ref{asm4}, we have 
	\begin{equation}\label{RR}
		\lim\limits_{k\to\infty}R_{k+1}=\lim\limits_{k\to\infty}\frac{1}{k}\sum\limits_{t=1}^k\phi_t\phi_t^{\tau}=G,
	\end{equation}
	where $G=c_0I$.	Then for any $\varepsilon>0$, the event $\|R_{k+1}-G\|\leq\varepsilon$ happens i.o. with probability 1.	
	
	Secondly, we show that 
	\begin{equation}\label{bound}
		\{\|\beta_{k+1}\|\leq m_0, k\geq 0\} \ \hbox{happens i.o. with probability 1.} 
	\end{equation}
	Let us denote $\Psi_k=\left[
	\begin{array}{ccc}\phi_0^{\tau} &\cdots &\phi_k^{\tau}\end{array}\right]^{\tau}$ and $Y_{k+1}=\left[\begin{array}{ccc}\bar y_1 &\cdots &\bar y_{k+1}\end{array}\right]^{\tau}$ with $\bar y_{k+1}$ defined in (\ref{bary}). By (\ref{RR}), we have 
	$$\frac{1}{k}\Psi_k^{\tau}\Psi_k=\frac{1}{k}\sum\limits_{t=0}^k\phi_k\phi_k^{\tau}\xrightarrow{k\to\infty} G,$$
	then by Lemma \ref{project}, we have for sufficiently large $k$,
	\begin{equation}\label{psi}
		Y_{k+1}^{\tau}\Psi_k(\Psi_k^{\tau}\Psi_k)^{-1}\Psi_k^{\tau}Y_{k+1}\leq Y_{k+1}^{\tau}Y_{k+1}.
	\end{equation}		
	By Algorithm \ref{alg1}, it is evident that
	$$
	\beta_{k+1}=P_{k+1}\sum\limits_{t=0}^k\phi_t\bar y_{t+1}=(P_0^{-1}+\Psi_k^{\tau}\Psi_k)^{-1}\Psi_k^{\tau}Y_{k+1},
	$$
	thus by (\ref{psi}), it follows that 
	\begin{equation}
		\begin{aligned}
			&\|\beta_{k+1}\|\\
			\leq&\|(P_0^{-1}+\Psi_k^{\tau}\Psi_k)^{-\frac{1}{2}}\|\|(P_0^{-1}+\Psi_k^{\tau}\Psi_k)^{-\frac{1}{2}}\Psi_k^{\tau}Y_{k+1}\|\\
			\leq&\|(\Psi_k^{\tau}\Psi_k)^{-\frac{1}{2}}\|\|Y_{k+1}\|=\|(\frac{1}{k}\Psi_k^{\tau}\Psi_k)^{-\frac{1}{2}}\|\|\frac{1}{\sqrt{k}}Y_{k+1}\|.
		\end{aligned}
	\end{equation}
	By the definition of $\bar y_{k+1}$ in (\ref{bary}), model (\ref{model2}), Assumptions \ref{asm2}-\ref{asm4} and the fact $|\tanh(\cdot)|\leq1$, we have
	$$
	\begin{aligned}		&\frac{1}{k}\|Y_{k+1}\|^2=\frac{1}{k}\sum\limits_{t=0}^k\bar y_{t+1}^2\leq\frac{2}{k}\sum\limits_{t=0}^k[(\beta^{*\tau}\phi_t)^2+w_{t+1}^2]\\
	&~~~~~~~~~~~~~~~~~~~~~~~~~~~~\xrightarrow{k\to\infty} 2c_0\|\beta^*\|^2+2\sigma^2.
	\end{aligned}	
	$$
	Thus we have
	\begin{equation}\label{m0}
		\limsup\limits_{k\to\infty}\|\beta_{k+1}\|\leq m_0, \ \hbox{a.s.,}
	\end{equation}
	and (\ref{bound}) holds.
	Therefore, by (\ref{RR}) and (\ref{bound}), we have $\{x_{k}\in \bar D\}$ happens i.o. with probability 1 and Condition B4) is verified. Then by Proposition \ref{ljungtheorem}, we have $x_k\to D_c$ as $k\to\infty$ almost surely and $D_c$ is defined in (\ref{fact}).
	
	We now prove that $\beta_k$ converges to a limit point $\beta^*$ or $-\beta^*$.
	
	We first show that $\beta_k$ will not converge to the point $0$. 
	By Algorithm \ref{alg1}, we can see that $\beta_{k}$ is a rational function of random variables $\{z_0,\phi_0,w_1, \cdots,z_{k-1},\phi_{k-1}, w_{k}\}$, which are jointly absolutely continuous with respect to Lebesgue measure by Assumptions \ref{asm2}-\ref{asm4}.
	From the results in \cite{meyn1985zero}, it follows that $\beta_{k}$ is also absolutely continuous with respect to Lebesgue measure.
	Then with the initial value of Algorithm \ref{alg1} satisfying that $\beta_0\ne0$, we have for any $k\geq1$,
	\begin{equation}\label{sign}
		P(\beta_{k}^{\tau}\beta^{*}\ne0)=1.
	\end{equation}
	Therefore, we can derive $\beta_k\nrightarrow0$ and thus $\beta_k$ will converge to the set $\{\beta^*,-\beta^*\}$ almost surely (The proof details are provided in Appendix \ref{app1}).
	
	We then prove that $\beta_k$ converges to a limit point.
	Denote $\mathcal{F}_k=\sigma\{\phi_t,z_t,w_t,t\leq k\}$ and 
	\begin{equation}\label{ee}
		e_{k+1}=y_{k+1}\tanh\big(\frac{\beta_{k}^{\tau}\phi_{k}y_{k+1}}{\sigma^2}\big)-\beta_{k}^{\tau}\phi_{k}.
	\end{equation}
	By model (\ref{model2}) and Assumption \ref{asm3}, we have
	\begin{equation}\label{e}
		\mathbb{E}\left[e_{k+1}^2|\mathcal{F}_k\right]\leq3[(\beta^{*\tau}\phi_k)^2+(\beta_k^{\tau}\phi_k)^2]+3\sigma^2.
	\end{equation}
	From (\ref{pk}) and (\ref{RR}), it follows that $\frac{1}{k}P_{k}^{-1}=\frac{1}{k}[\sum\nolimits_{t=1}^{k}\phi_t\phi_t+P_0]\to G>0$, thus we have $\|P_k\|=O(\frac{1}{k})$ and then $\|P_{k+1}-P_k\|=\|P_{k+1}(P_{k+1}^{-1}-P_k^{-1})P_k\|=\|P_{k+1}\phi_k\phi_k^{\tau}P_k\|=O(\frac{\|\phi_k\|^2}{k^2})$.
	Hence, by (\ref{pk}), we have
	$$\begin{aligned}
		&a_k\|P_k\phi_k\|^2=tr(a_kP_k\phi_k\phi_k^{\tau}P_k)\\
		=&tr(P_{k+1}-P_k)=O(\frac{\|\phi_k\|^2}{k^2}).
	\end{aligned}$$
	Moreover, since $\|\phi_k\|^4$ is u.i. by Assumption \ref{asm4}, we have $\sup_{k\geq1}\mathbb{E}\left[\|\phi_k\|^4\right]<\infty$.
	Thus by (\ref{beta}), (\ref{m0}), (\ref{e}), we obtain
	$$
		\begin{aligned}
			&\sum\limits_{k=1}^{\infty}\mathbb{E}[\|a_kP_k\phi_ke_{k+1}\|^2]
			=\sum\limits_{k=1}^{\infty}\mathbb{E}[\mathbb{E}[\|a_kP_k\phi_ke_{k+1}\|^2|\mathcal{F}_k]]\\
			\leq&3\sum\limits_{k=1}^{\infty}\mathbb{E}\left[a_k\|P_k\phi_k\|^2[(\beta^{*\tau}\phi_k)^2+(\beta_k^{\tau}\phi_k)^2+\sigma^2]\right]\\
			=&O(\sum\limits_{k=1}^{\infty}\frac{\mathbb{E}\left[\|\phi_k\|^4\right]}{k^2})<\infty.
		\end{aligned}$$
	Since for any sequence of random variables $Z_k$, $\sum\limits_{k=1}^{\infty}\mathbb{E}\left[|Z_k|\right]<\infty$ implies $\sum\limits_{k=1}^{\infty}Z_k<\infty$ \cite{stout1974almost}, we have that $\sum\limits_{k=1}^{\infty}\|a_kP_k\phi_ke_{k+1}\|^2$ converges a.s. Thus by (\ref{beta}) and (\ref{ee}), we have
	\begin{equation}
		\begin{aligned}
			&\sum\limits_{k=1}^{\infty}\|\beta_{k+1}-\beta_k\|^2=\sum\limits_{k=1}^{\infty}\|a_kP_k\phi_ke_{k+1}\|^2<\infty.\\
		\end{aligned}	
	\end{equation}
	Hence, $\lim\limits_{k\to\infty}\|\beta_{k+1}-\beta_k\|^2=0$, which means that $\beta_k$ cannot jump from a small neighborhood of $\beta^*$ to a small neighborhood of $-\beta^*$ infinite times. Consequently, $\beta_k$ will converge to a limit point which is either $\beta^*$ or $-\beta^*$ almost surely.
	
	We now proceed to show that our algorithm can handle the case $\beta^*=0$.
	At this time, $b_1^*(t)=b_2^*(t)\equiv0$. Then following the analysis method similar to the case $x(0)\in D_{A,3}$, we can prove the assertion (\ref{fact}) for $\beta^*=0$. Moreover, Condition B4) holds for $\bar D$ in (\ref{bard1}) for $\beta^*=0$. Therefore, by Proposition \ref{ljungtheorem}, we have $\beta_k\to0$ as $k\to\infty$ almost surely.
	
	Therefore, we complete the proof of Theorem \ref{odetheorem1}.$\hfill\blacksquare$      
	\subsection{Proof of Theorem 2}
	Firstly, we prove the inequality (\ref{them4.21}).
	Without loss of generality, we assume that $\{\phi_k, y_{k+1}\}$ is generated by the sub-model $y_{k+1}=\beta^{*\tau}\phi_k+w_{k+1}$. 
	We now show that if $\lim\limits_{k\to\infty}\beta_k=\beta^*$, then (\ref{them4.21}) holds.
	Denote the conditional p.d.f of $\phi_k$ given $\beta_k$ as $g_k(\phi_k|\beta_k)$. 
	From the fact that $\phi_k\to\phi$ in distribution, the equi-continuity and the convergence of $g_k(\phi_k|\beta_k)$ at the point $\beta_k=\beta^*$, we have $g_k(\phi_k|\beta_k)\to\bar g(\phi)$.
	Let us denote $M_k(\beta_k,\phi,w)=(1/\sigma)\mathbb{I}_{\{\beta_k^{\tau}\phi(\beta^{*\tau}\phi+ w)\geq0\}}\Phi'(w/\sigma)g_k(\phi|\beta_k)$ and $M^*(\beta^*,\phi,w)=(1/\sigma)\mathbb{I}_{\{\beta^{*\tau}\phi(\beta^{*\tau}\phi+ w)\geq0\}}\Phi'(w/\sigma)\bar g(\phi)$. 
	By Lebesgue dominated convergence theorem, the probability that $\{\phi_k,y_{k+1}\}$ is categorized correctly is calculated as follow:  
	\begin{equation}\label{pp}
		\begin{aligned}
			&\lim\limits_{k\to\infty}P\big((y_{k+1}-\beta_{k}^{\tau}\phi_k)^2\leq(y_{k+1}+\beta_{k}^{\tau}\phi_k)^2|\beta_k\big)\\
			=&\lim\limits_{k\to\infty}P\big(\beta_k^{\tau}\phi_k(\beta^{*\tau}\phi_k+ w_{k+1})\geq0|\beta_k\big)\\
			=&\lim\limits_{k\to\infty}\int_{\mathbb{R}^d}\int_{\mathbb{R}}M_k(\beta_k,\phi_k,w_{k+1})dw_{k+1}d\phi_k\\
			=&\lim\limits_{k\to\infty}\int_{\mathbb{R}^d}\int_{\mathbb{R}}M_k(\beta_k,\phi,w)dwd\phi\\
			=&\int_{\mathbb{R}^d}\int_{\mathbb{R}}M^*(\beta^*,\phi,w)dwd\phi=P\big((\beta^{*\tau}\phi)^2+\beta^{*\tau}\phi w\geq0\big)\\
			=&1-\mathbb{E}\left[\Phi\left(-\frac{|\beta^{*\tau}\phi|}{\sigma}\right)\right]\geq1-\mathbb{E}\left[\exp\left(-\frac{(\beta^{*\tau}\phi)^2}{2\sigma^2}\right)\right].
		\end{aligned}
	\end{equation}
	If $\lim\limits_{k\to\infty}\beta_k=-\beta^*$, by a similar analysis as that of (\ref{pp}), we also have
	$\lim\limits_{k\to\infty}P((y_{k+1}+\beta_{k}^{\tau}\phi_k)^2\leq(y_{k+1}-\beta_{k}^{\tau}\phi_k)^2|\beta_k)
	\geq1-\mathbb{E}\left[\exp\left(-\frac{(\beta^{*\tau}\phi)^2}{2\sigma^2}\right)\right].$
	Hence, the inequality (\ref{them4.21}) is obtained.	
	
	Secondly, we prove the inequality (\ref{them4.22}).		
	We now show that if $\lim\limits_{k\to\infty}\beta_k=\beta^*$, (\ref{them4.22}) holds.
	Denote
	$$
	\begin{aligned}
		&\mathcal{A}_{k,1}=\{\omega: y_{k+1}=\beta^{*\tau}\phi_k+w_{k+1}\},\\
		&\mathcal{A}_{k,2}=\{\omega: y_{k+1}=-\beta^{*\tau}\phi_k+w_{k+1}\},\\
		&\mathcal{A}_{k,3}=\{\omega: (y_{k+1}+\beta_k^{\tau}\phi_k)^2\leq(y_{k+1}-\beta_k^{\tau}\phi_k)^2\}\cap\mathcal{A}_{k,1},\\
		&\mathcal{A}_{k,4}=\{\omega: (y_{k+1}-\beta_k^{\tau}\phi_k)^2\leq(y_{k+1}+\beta_k^{\tau}\phi_k)^2\}\cap\mathcal{A}_{k,2},
	\end{aligned}
	$$
	where $\mathcal{A}_{k,1}$, $\mathcal{A}_{k,2}$ denote the events that the data $\{\phi_k,y_{k+1}\}$ is generated by these two sub-models, $\mathcal{A}_{k,3}$, $\mathcal{A}_{k,4}$ represent the events that the data $\{\phi_k,y_{k+1}\}$ is categorized into the wrong cluster, and thus $\mathcal{A}_{k,1}-\mathcal{A}_{k,3}$, $\mathcal{A}_{k,2}-\mathcal{A}_{k,4}$ are the events that the data is categorized into the correct cluster.
	Then the evaluation index (\ref{wce}) can be written as follows:
	\begin{equation}\label{Jn}
		\begin{aligned}
			J_n=&\frac{1}{n}\sum\limits_{k=1}^n(y_{k+1}-\beta_k^{\tau}\phi_k)^2\left[\mathbb{I}_{\{\mathcal{A}_{k,1}-\mathcal{A}_{k,3}\}}+\mathbb{I}_{\mathcal{A}_{k,4}}\right]\\
			&+\frac{1}{n}\sum\limits_{k=1}^n(y_{k+1}+\beta_k^{\tau}\phi_k)^2\left[\mathbb{I}_{\{\mathcal{A}_{k,2}-\mathcal{A}_{k,4}\}}+\mathbb{I}_{\mathcal{A}_{k,3}}\right]\\
			=&L_{n,1}+L_{n,2}+L_{n,3},
		\end{aligned}
	\end{equation}
where $L_{n,1}=\frac{1}{n}\sum\limits_{k=1}^n\big\{(y_{k+1}-\beta_k^{\tau}\phi_k)^2\mathbb{I}_{\mathcal{A}_{k,1}}+(y_{k+1}+\beta_k^{\tau}\phi_k)^2\mathbb{I}_{\mathcal{A}_{k,2}}\big\}$, $L_{n,2}=\frac{1}{n}\sum\limits_{k=1}^n\big\{(y_{k+1}+\beta_k^{\tau}\phi_k)^2-(y_{k+1}-\beta_k^{\tau}\phi_k)^2\big\}\mathbb{I}_{\mathcal{A}_{k,3}}$ and $L_{n,3}=\frac{1}{n}\sum\limits_{k=1}^n\big\{(y_{k+1}-\beta_k^{\tau}\phi_k)^2-(y_{k+1}+\beta_k^{\tau}\phi_k)^2\big\}\mathbb{I}_{\mathcal{A}_{k,4}}$.
	We now analyze the RHS of (\ref{Jn}) term by term.	
	For the term $L_{n,1}$, we have the following expression:
	$$
	\begin{aligned}
		L_{n,1}=&\frac{1}{n}\sum\limits_{k=1}^n(\tilde\beta_k^{\tau}\phi_k)^2-\frac{2}{n}\sum\limits_{k=1}^n\tilde\beta_k^{\tau}\phi_kw_{k+1}\mathbb{I}_{\mathcal{A}_{k,1}}\\
		&+\frac{2}{n}\sum\limits_{k=1}^n\tilde\beta_k^{\tau}\phi_kw_{k+1}\mathbb{I}_{\mathcal{A}_{k,2}}+\frac{1}{n}\sum\limits_{k=1}^nw_{k+1}^2,
	\end{aligned}			
	$$
	where $\tilde\beta_k=\beta_k-\beta^*$. By Assumption \ref{asm3},
	$\lim\limits_{n\to\infty}\frac{1}{n}\sum\limits_{k=1}^nw_{k+1}^2=\sigma^2.$ 
	By $\lim\limits_{k\to\infty}\tilde\beta_{k}=0$ and the average boundedness of $\|\phi_k\|^2$ and $\|\phi_kw_{k+1}\|$ from Lemma \ref{dk}, we obtain
	\begin{equation}\label{l1}
		\lim\limits_{n\to\infty} L_{n,1}=\sigma^2, \ \hbox{a.s.}
	\end{equation}	
	For the term $L_{n,2}$, let us denote $\mathcal{A}_1=\{(\beta^{*\tau}\phi)^2+\beta^{*\tau}\phi w\leq0\}$ and $\mathcal{A}_2=\{y=\beta^{*\tau}\phi+w\}$, by Assumptions \ref{asm2}-\ref{asm4}, we have $P\left(\mathcal{A}_1\cap\mathcal{A}_2|\phi\right)=P\left(\mathcal{A}_1|\phi\right)P\left(\mathcal{A}_2|\phi\right)=p\Phi\big(-\frac{|\beta^{*\tau}\phi|}{\sigma}\big)$. 
	Besides, from Lemma \ref{dk}, $\left[\beta^{*\tau}\phi_k\phi_k^{\tau}\beta^*+\beta^{*\tau}\phi_kw_{k+1}\right]\mathbb{I}_{\mathcal{A}_1\cap\mathcal{A}_2}$ is asymptotically stationary and ergodic.
	Thus by Assumptions \ref{asm3}-\ref{asm4} and (\ref{pp}), we obtain
	\begin{equation}\label{l2}
		\begin{aligned}
			&\lim\limits_{n\to\infty} L_{n,2}=\lim\limits_{n\to\infty}\frac{4}{n}\sum\limits_{k=1}^n\left[\beta_k^{\tau}\phi_k\phi_k^{\tau}\beta^*+\beta^{*\tau}\phi_kw_{k+1}\right]\mathbb{I}_{\mathcal{A}_{k,3}}\\
			=&4\mathbb{E}\left[(\beta^{*\tau}\phi)^2\mathbb{E}\left[\mathbb{I}_{\mathcal{A}_1\cap\mathcal{A}_2}|\phi\right]\right]+4\mathbb{E}\left[\beta^{*\tau}\phi\mathbb{E}\left[ w\mathbb{I}_{\mathcal{A}_1\cap\mathcal{A}_2}|\phi\right]\right]\\
			=&4\mathbb{E}\big[(\beta^{*\tau}\phi)^2P\left(\mathcal{A}_1\cap\mathcal{A}_2|\phi\right)\big]\\
			&+4p\mathbb{E}\bigg[\frac{|\beta^{*\tau}\phi|}{\sqrt{2\pi}\sigma}\int_{-\infty}^{-|\beta^{*\tau}\phi|}w\exp\left(-\frac{w^2}{2\sigma^2}\right)dw\bigg]\\
			=&4p\mathbb{E}\left[\eta(\phi)\right].
		\end{aligned}
	\end{equation}
	Similarly, for the term $L_{n,3}$, we have
	\begin{equation}\label{l3}
		\lim\limits_{n\to\infty} L_{n,3}=4(1-p)\mathbb{E}\left[\eta(\phi)\right].
	\end{equation}		
	Furthermore, since for any positive constant $a$, $\int_{-\infty}^{-a}a\exp\big(-\frac{x^2}{2\sigma^2}\big)dx\leq\sigma^2\exp\big(-\frac{a^2}{2\sigma^2}\big)$, we have $\eta(\phi)\leq0$. 
	Thus by (\ref{Jn})-(\ref{l3}), (\ref{them4.22}) is obtained.	
	Similarly, if $\lim\limits_{k\to\infty}\beta_k=-\beta^*$, we also have (\ref{them4.22}).
	Therefore, we complete the proof.	 $\hfill\blacksquare$
	
	\subsection{Proofs of Theorem 3 and Theorem 4}	
	
	Denote $x_k=[\theta_{k,1}^{\tau}, \theta_{k,2}^{\tau}, \text{vec}^{\tau}(R_k)]^{\tau}$, similar to (\ref{recf}), it is not difficult to obtain that $x_k$ evolves according to the following dynamical systems:
	\begin{equation}\label{rec}
		x_{k+1}=x_k+\frac{1}{k}Q(x_k,\phi_k,y_{k+1}),
	\end{equation}
	where $Q(x_k,\phi_k,y_{k+1})$ is determined via Algorithm \ref{alg2}.
	In order to analyze (\ref{rec}) using the ODE method, we introduce the corresponding ODEs as follows:
\begin{subequations}\label{ode_2}
	\begin{align}
		\frac{d}{dt}\theta_1(t)&=R^{-1}(t)f_1(\theta_1(t)),\label{f1}\\
		\frac{d}{dt}\theta_2(t)&=R^{-1}(t)f_2(\theta(t)),\label{f2}\\
		\frac{d}{dt}R(t)&=G-R(t)\label{gg},
	\end{align}	
\end{subequations}
	where $f_1(\theta_1(t))=\lim\limits_{k\to\infty}\mathbb{E}\left[\phi_k\left(y_{k+1}-\theta_1^{\tau}(t)\phi_k\right)\right]$, $f_2(\theta(t))=\lim\limits_{k\to\infty}\mathbb{E}\left[\phi_k\left(m_{k+1}\tanh\left(\frac{\theta_2^{\tau}(t)\phi_km_{k+1}}{\sigma^2}\right)-\theta_2^{\tau}(t){\phi}_{k}\right)\right]$, $\theta(t)=(\theta_1(t),\theta_2(t))$ and $G=\lim\limits_{k\to\infty}\mathbb{E}\left[\phi_k\phi_k^{\tau}\right]$.
	
	Before giving the proof of Theorem \ref{odetheorem2}, a related lemma is given as follows:
	\begin{lemma}\label{regular1}
		Under Assumptions \ref{asm2}$'$ and \ref{asm3}-\ref{asm4}, Conditions B1)-B3) in Proposition \ref{ljungtheorem} are all satisfied in the open area $D=\{x:R>0\}$, where $x=[\theta_1^{\tau}, \theta_2^{\tau}, \text{vec}^{\tau}(R)]^{\tau}$.
	\end{lemma}
	
	\noindent\hspace{1em}{\textbf{\itshape Proof of Theorem \ref{odetheorem2}:}}		
	We will analyze the convergence of $\theta_{k,1}$ and $\theta_{k,2}$ separately by verifying all conditions of Proposition \ref{ljungtheorem}. 
	By Lemma \ref{regular1}, it remains to prove the requirements on the trajectories of ODEs (\ref{ode_2}) and Condition B4).
	
	\textbf{Step 1: Convergence analysis of the sequence $\{\theta_{k,1}\}$.}
	
	It is clear that the trajectory generated by (\ref{gg}) is the same as that of the ODE (\ref{odee}), which evolves according to (\ref{R11}).	
	We now establish the stability result of (\ref{f1}). For this, let us construct the Lyapunov function $V_1(\theta_1(t),R(t))=\frac{1}{2}\tilde\theta_1^{\tau}(t)R(t)\tilde\theta_1(t)$ with $\tilde\theta_{1}(t)=\theta_{1}(t)-\theta_1^*$. Then we have
	\begin{equation}\label{vvv1}
		\begin{aligned}
			\frac{dV_1(\theta_1(t),R(t))}{dt}
			=\tilde\theta_1^{\tau}(t)f_1(\theta_1(t))+\frac{1}{2}\tilde\theta_1^{\tau}(t)[G-R(t)]\tilde\theta_1(t).
		\end{aligned}	
	\end{equation}
	For the first term on the RHS of (\ref{vvv1}), by (\ref{bm}), (\ref{ode_2}), we have
	\begin{equation}
		\begin{aligned} &\tilde\theta_1^{\tau}(t)f_1(\theta_1(t))=\lim\limits_{k\to\infty}\mathbb{E}\big[\tilde\theta_1^{\tau}(t)\phi_k\big(y_{k+1}-\theta_1^{\tau}(t)\phi_k\big)\big]\\ =&\lim\limits_{k\to\infty}\mathbb{E}\left[\tilde\theta_1^{\tau}(t)\phi_k\big(\mathbb{E}\left[y_{k+1}|\phi_k\right]-\theta_1^{\tau}(t)\phi_k\big)\right]\\				=&-\tilde\theta_1^{\tau}(t)\lim\limits_{k\to\infty}\mathbb{E}\left[\phi_k\phi_k^{\tau}\right]\tilde\theta_1(t)=-c_0\|\tilde\theta_1(t)\|^2,
		\end{aligned}	
	\end{equation}
	where the last inequality holds by Assumption \ref{asm4} and $c_0$ is a positive constant defined in (\ref{R11}).
	Then it follows that 
	$$
	\frac{d}{dt}V_1(\theta_1(t),R(t))=-2V_1(\theta_1(t),R(t))+r(t),
	$$	
	where $r(t)=\frac{1}{2}\tilde\theta_1^{\tau}(t)\left(R(t)-G\right)\tilde\theta_1(t)$ tends to 0 by (\ref{R11}).
	Thus by Lemma \ref{vvv}, we have $\lim\limits_{t\to\infty}V(\theta_1(t),R(t))=0$.
	From the positive-definiteness property of $R(t)$, it follows that $\lim\limits_{t\to\infty}\theta_1(t)=\theta_1^*$. Therefore, we obtain that the ODE (\ref{f1}) has the invariant set $D'_{c,1}=\{[\theta_1^{*\tau},\text{vec}^{\tau}(G)]\}$ with the domain of attraction $D'=\{v=[\theta^{\tau},\text{vec}^{\tau}(R)]: R>0\}$.
	
	Denote $\bar D'=\{v=[\theta^{\tau},\text{vec}^{\tau}(R)]:\|\theta\|\leq m_1, \varepsilon_1I\leq R\leq\varepsilon_2I\}$ with $m_1=\sqrt{3c_0(\|\beta_1^*\|^2+\|\beta_2^*\|^2)+3\sigma^2}$ and  $0<\varepsilon_1<\varepsilon_2$. It is clear that the trajectories of (\ref{ode_2}) that starts in $\bar D'$ remains in $\bar D'$. Besides, similar to (\ref{m0}), we have $\limsup\nolimits_{k\to\infty}\|\theta_{k+1,1}\|\leq m_1, \ \hbox{a.s.}$
	Thus Condition B4) is verified. By Proposition \ref{ljungtheorem}, we have
	\begin{equation}\label{theta1}
		\lim\limits_{k\to\infty}\theta_{k,1}=\theta_1^*,\ \hbox{a.s.}
	\end{equation}
	
	\textbf{Step 2:  Convergence analysis of the sequence $\{\theta_{k,2}\}$.}
	
	The proof is similar to that of Theorem \ref{odetheorem1} and we only need to establish the stability results of the ODEs (\ref{f2}). We just provide the analysis for $\theta_2(0)\in D_{A,1}$, and omit the analysis for $\theta_2(0)\in D_{A,2}$ and $\theta_2(0)\in D_{A,3}$. 
	
	For this, consider the Lyapunov function $V_2(\theta_2(t),R(t))=\frac{1}{2}\tilde\theta_2(t)^{\tau}R(t)\tilde\theta_2(t)$ with $\tilde\theta_{2}(t)=\theta_{2}(t)-\theta_2^*$. Then we have
	\begin{equation}\label{v}
		\begin{aligned}
			\frac{dV_2(\theta_2(t),R(t))}{dt}
			=\tilde\theta_2^{\tau}(t)f_2(\theta(t))+\frac{1}{2}\tilde\theta_2^{\tau}(t)[G-R(t)]\tilde\theta_2(t).
		\end{aligned}	
	\end{equation}
	We now analyze the first term on the RHS of (\ref{v}). 
	Denote $m^*_{k+1}=y_{k+1}-\theta_1^{*\tau}\phi_k$, then by (\ref{model1}), we have $m^*_{k+1}=z_k\theta_2^{*\tau}\phi_k+w_{k+1}$. Moreover, by (\ref{ode_2}), we obtain
	\begin{equation}\label{f22}
		f_2(\theta(t))=h_1(\theta(t))+h_2(\theta(t)),
	\end{equation}
	where $h_1(\theta(t))=\lim\limits_{k\to\infty}\mathbb{E}\big[{\phi}_{k}\big(m^*_{k+1}\tanh\big(\frac{\theta_{2}^{\tau}(t)\phi_km^*_{k+1}}{\sigma^2}\big)-\theta_{2}^{\tau}(t){\phi}_{k}\big)\big]$ and $h_2(\theta(t))=f_2(\theta(t))-h_1(\theta(t))$.
	Since the analysis of $\tilde\theta_2^{\tau}(t)h_1(\theta(t))$ is the same as that of $\tilde\beta^{\tau}(t)f(\beta(t))$ in (\ref{rdv}), from (\ref{dvbeta}) in Theorem \ref{odetheorem1}, we have
	\begin{equation}\label{vv2}
		\tilde\theta_2^{\tau}(t)h_1(\theta(t))\leq-C\|\tilde\theta_2(t)\|^2+Cb_2^{*2}(t)+1.2\|\tilde\theta_2(t)\||b_2^*(t)|,
	\end{equation}
	where $C$ is a positive constant defined in (\ref{dvf4}) and the term $b_2^*(t)$ will tend to zero.	
	Let $\tilde\theta_{k,1}=\theta_{k,1}-\theta_1^*$. By mean-value theorem, Schwarz inequality and the fact that $|\tanh(z)+z\tanh'(z)|\leq1.2$, we obtain that
	\begin{equation}\label{vv3}
		\begin{aligned}
			&\tilde\theta_2^{\tau}(t)h_2(\theta(t))=\tilde\theta_2^{\tau}(t)\left(f_2(\theta(t))-h_1(\theta(t))\right)\\
			=&\lim\limits_{k\to\infty}\mathbb{E}\big[\tilde\theta_2^{\tau}(t){\phi}_{k}\phi_k^{\tau}\tilde\theta_{k,1}\big[\tanh(\bar l_k(t))+\bar l_k(t)\tanh'(\bar l_k(t))\big]\big]\\
			\leq& c_2\|\tilde\theta_2(t)\|\lim\limits_{k\to\infty}\big(\mathbb{E}\|\tilde\theta_{k,1}\|^2\big)^{1/2},
		\end{aligned}
	\end{equation}
	where $\bar l_k(t)=\frac{\theta_{2}^{\tau}(t)\phi_k(y_{k+1}-\zeta_k(t))}{\sigma^2}$, $\zeta_k(t)$ is between $\theta_1^{*\tau}\phi_k$ and $\theta_1^{\tau}(t)\phi_k$, and $c_2=1.2\sqrt{\mathbb{E}\|\phi\|^4}$.		
	Then by (\ref{v})-(\ref{vv3}) and $G=c_0I$, we obtain
	\begin{equation}
	\frac{dV_2(\theta_2(t),R(t))}{dt}\leq-Cc_0^{-1}V_2(\theta_2(t),R(t))+r(t),	
	\end{equation}
	where $r(t) =Cb_2^{*2}(t)+1.2\|\tilde\theta_2(t)\||b_2^*(t)|+c_2\|\tilde\theta_2(t)\|\lim\limits_{k\to\infty}(\mathbb{E}\|\tilde\theta_{k,1}\|^2)^{1/2}+(\frac{1}{2}-Cc_0^{-1})\tilde\theta_2^{\tau}(t)(G-R(t))\tilde\theta_2(t)$. By (\ref{R11}), (\ref{theta1}) and the fact $b_2^*(t)\to0$, we have $\lim\limits_{t\to\infty}r(t)=0$.
	Thus by Lemma \ref{vvv}, it follows that $\lim\limits_{t\to\infty}V_2(\theta_2(t),R(t))=0$, and by the positiveness-definite property of $R(t)$ from (\ref{R11}), we have $\lim\limits_{t\to\infty}\theta_2(t)=\theta_2^*$, and the assertion (\ref{fact}) (replacing $\beta^*$ with $\theta_2^*$) holds.
	
	Then by Proposition \ref{ljungtheorem}, we have
	\begin{equation}\label{theta2}
		\lim\limits_{k\to\infty}\theta_{k,2}=\theta_2^*,\ \hbox{or}\ \lim\limits_{k\to\infty}\theta_{k,2}=-\theta_2^*,\ \hbox{a.s.}
	\end{equation}		
	Thus, the results of Theorem \ref{odetheorem2} can be obtained.
	$\hfill\blacksquare$
	
	\noindent\hspace{1em}{\textbf{\itshape Proof of Theorem \ref{odetheorem4}:}}
	The proof is similar to the way used in Theorem \ref{dcp}, which is omitted.
	$\hfill\blacksquare$
	\section{Simulation Results}\label{simu}
	
	In this section, we conduct simulations for the general asymmetric MLR problem to verify the effectiveness of our algorithm.
	Consider the data $\{\phi_k,y_{k+1}\}_{k=1}^{\infty}$ are generated by the following dynamical model:
	$$\begin{aligned}
		y_{k+1}&=z_{k}\beta_1^{*\tau}\phi_k+(1-z_{k})\beta_2^{*\tau}\phi_k+w_{k+1},\\
		\phi_{k+1}&=0.5\phi_k+e_{k+1},		
	\end{aligned}
	$$
	where $\phi_k\in\mathbb{R}^{3}$, $\beta_1^*=[1~15~13]^{\tau}$, $\beta_2^*=[-10~-11~-12]^{\tau}$, $z_k\stackrel{\text{i.i.d}}{\sim}P(z_k=0)=P(z_k=1)=0.5$,  $e_{k+1}\stackrel{\text{i.i.d}}{\sim} \mathcal{N}(0,I_3)$ and $w_{k+1}\stackrel{\text{i.i.d}}{\sim}\mathcal{N}(0,1)$.
	It is clear that the regressor $\{\phi_k\}$ is dependent, and all assumptions in Theorem \ref{odetheorem2} are satisfied.
	
	Firstly, we conduct Algorithm \ref{alg2} to estimate the unknown parameters $\beta_1^{*}$ and $\beta_2^{*}$, where the initial values are chosen as $\theta_{0,1}=[15~20~100]^{\tau}$, $\theta_{0,2}=[-42~-35~-30]^{\tau}$ and $P_0=I_3$.	
	The estimation error is defined by $\tilde\beta_{k,i}=\beta_{k,i}-\beta_{\mathcal{\bar I}_{k,i}}^{*}, (i=1,2)$ with $\mathcal{\bar I}_{k,i}=\arg\min_{j=1,2}\{\|\beta_{k,i}-\beta_{j}^{*}\|\}$.  
	From Fig.\ref{fig1}, one can see that estimation errors $\tilde\beta_{k,i}(i=1,2)$ tend to zero along the time $k$, and the within-cluster error (\ref{wce1}) also tends to a small value along the time $k$.
	Moreover, we note that $\mathcal{\bar I}_{k,i}, (i=1,2)$ are convergent, i.e., $(\beta_{k,1},\beta_{k,2})$ will converge to a limit point belonging to the set $\{(\beta_1^*,\beta_2^*),(\beta_2^*,\beta_1^*)\}$, which demonstrates the effectiveness of our algorithm.
	\begin{figure}[H]
		\centering
		\subfigure{\includegraphics[width=0.49\hsize]{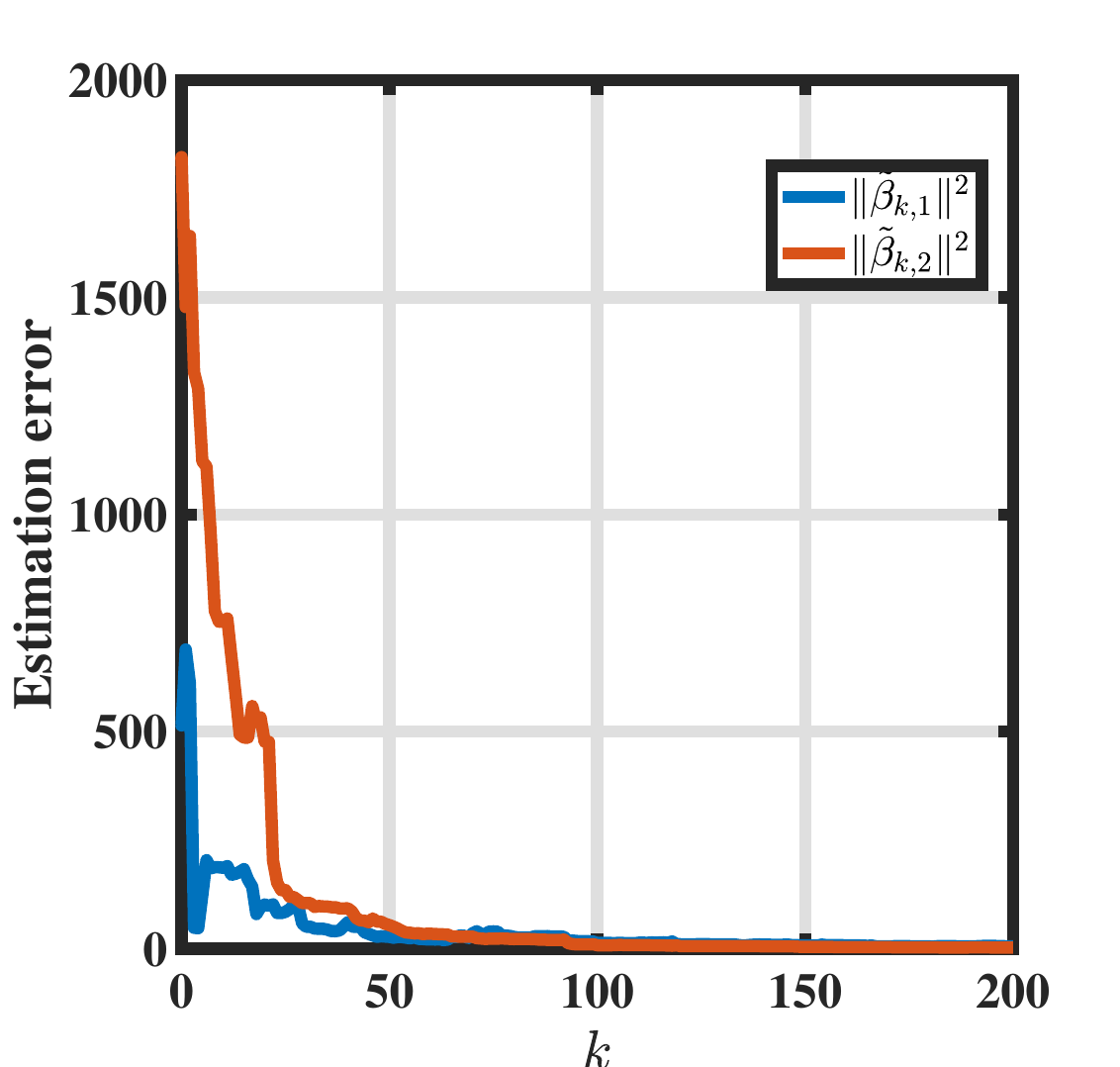}}
		\subfigure{\includegraphics[width=0.49\hsize]{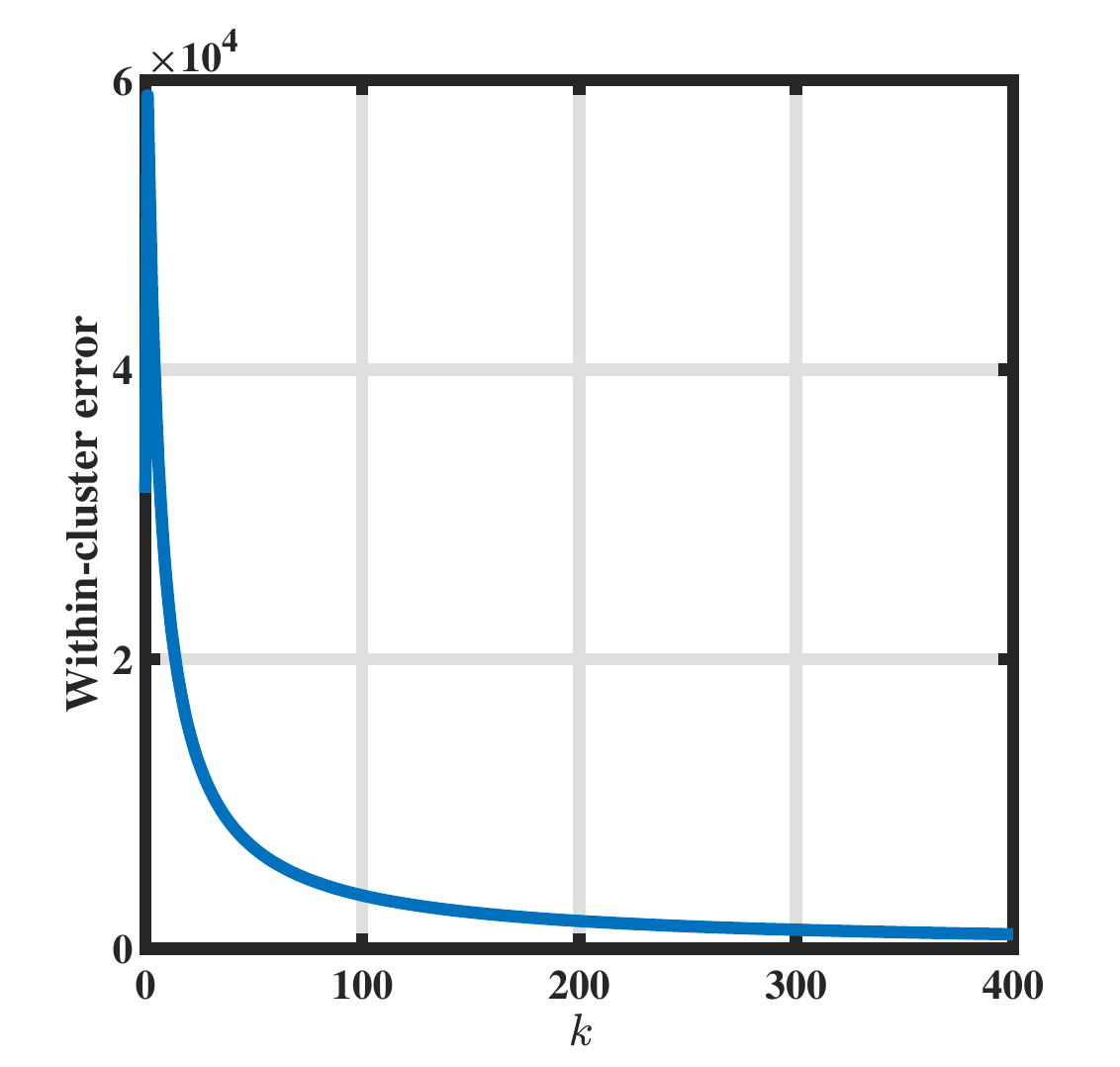}}
		\caption{Estimation error and clustering performance under Algorithm \ref{alg2}.}
		\label{fig1}
	\end{figure}
	Secondly, we compare the performance of Algorithm \ref{alg2} with the population EM algorithm, which is employed in most investigations for the MLR problem. The population EM with the finite number of samples \cite{balakrishnan2017statistical} is executed as follows:
	$$
	\begin{aligned}
		&\text{E-step:}\\
		&~~\alpha_{k,t}^i=\frac{\exp\left(-\frac{(y_{k+1}-\beta_{t,i}^{\tau}\phi_k)^2}{2\sigma^2}\right)}{\exp\left(-\frac{(y_{k+1}-\beta_{t,1}^{\tau}\phi_k)^2}{2\sigma^2}\right)+\exp\left(-\frac{(y_{k+1}-\beta_{t,2}^{\tau}\phi_k)^2}{2\sigma^2}\right)},\\
		&\text{M-step:}\\
		&~~\beta_{t+1,i}=\bigg(\frac{1}{n}\sum\limits_{k=1}^n\alpha_{k,t}^i\phi_k\phi_k^{\tau}\bigg)^{-1}\bigg(\frac{1}{n}\sum\limits_{k=1}^n\alpha_{k,t}^i\phi_k y_{k+1}\bigg),\\
	\end{aligned}
	$$  
	where $i=1,2$, $\beta_{t,i}$ is the estimate of parameter $\beta_i^*$ at the iteration step $t$, $\alpha_{k,t}^i$ is the conditional probability of $\{\phi_k,y_{k+1}\}$ belongs to $i$-th sub-model based the current estimate $\beta_{t,i}$, $\{\phi_k, y_{k+1}\}_{k=1}^n$ are the collection of $n$ samples, and $n$ is often chosen sufficiently large to approximate the population EM.
	
	In our simulation of the population EM algorithm, we choose the number of samples $n$ to be $5000$ and the total iteration step $T$ to be $20$.
	Both Algorithm \ref{alg2} and the population EM algorithm are initialized with the same values.
	Specifically, for $i=1,2$ and $j=1,2,3$, $\beta_{0,i}^{j}$ is sampled from a uniform distribution $U(\beta_{i}^{*j}-\kappa,\beta_{i}^{*j}+\kappa)$, where $\beta_{0,i}^{j}$ and $\beta_{i}^{*j}$ are the $j$-th element of $\beta_{0,i}$ and $\beta_{i}^{*}$, receptively.
	It can be seen that the parameter $\kappa$ measures the region of the initial values.	
	For each simulation with a given $\kappa$ in $[0,20]$, we run 500 independent realizations and plot the convergence probability of the algorithms, i.e., the proportion of 500 simulations that converge to true parameters, about the parameter $\kappa$ in Fig.\ref{fig2}.	
	From the simulation results, we see that the convergence probability of our algorithm does not depend on the parameter $\kappa$, while the convergence probability of the population EM algorithm will decrease to zero as $\kappa$ increases.
	These results show that the estimates generated by Algorithm \ref{alg2} will converge to the set $\left\{\left(\beta_1^*,\beta_2^*\right),\left(\beta_2^*,\beta_1^*\right)\right\}$ for any non-zero initial values, while the population EM algorithm does not. 
	\begin{figure}[H]
		\centering
		\includegraphics[width=0.9
		\hsize]{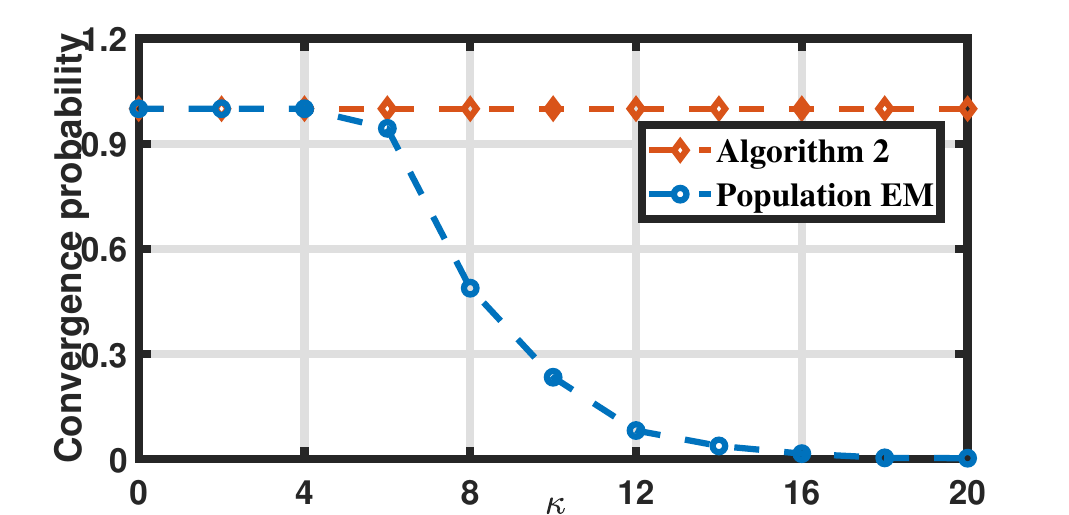}\label{fig}
		\caption{Comparison of convergence probabilities of Algorithm \ref{alg2} and the population EM algorithm.}
		\label{fig2}
	\end{figure}
	
\section{Conclusion}\label{conclusion}	
In this paper, we have investigated the online identification and data clustering problems of two classes of MLRs. 
For the symmetric MLR problem, we have proposed an online identification algorithm based on the EM principle, and established the global convergence result of the algorithm without imposing i.i.d data assumption for the first time. 
For the general asymmetric MLR problem, we have provided a two-step online identification algorithm by separately estimating two parts of the model, and also obtain the corresponding global convergence result. 
To the best of our knowledge, there is no theoretical result for the global convergence of the general asymmetric MLR model, even under i.i.d Gaussian assumptions on the regressors.
Based on the estimates of unknown parameters, we have shown that the performance of data clustering is asymptotically the same as the case where the true parameters are known. 
For further investigation, many interesting problems need to be studied, e.g., how to relax the asymptotically stationary and ergodic assumptions on the regressors, how to establish global convergence of unbalanced general asymmetric MLR model, and how to analyze mixed nonlinear regressions.

\appendices
\section{}\label{app1}
\noindent\hspace{1em}{\itshape Proof of (\ref{them4.23}):}
By Assumption \ref{asm4} and the property of Gaussian distribution, we have that $\beta^{*\tau}\phi\sim\mathcal{N}(0,\beta^{*\tau}\Sigma\beta^*)$. By simple calculations, we have
$
\mathbb{E}\big[\exp\big(-\frac{(\beta^{*\tau}\phi)^2}{2\sigma^2}\big)\big]=\frac{\sigma}{\sqrt{\sigma^2+\beta^{*\tau}\Sigma\beta^*}}.
$
Thus (\ref{them4.23}) holds. 	
$\hfill\blacksquare$

\noindent\hspace{1em}{\itshape Proof of Lemma \ref{dk}:}
For the first statement, by Assumptions \ref{asm2}-\ref{asm4}, from the stationary and mutual independent properties of $z_k$, $\phi_k$ and $w_{k+1}$, it follows that $\{d_{k}\}$ approaches a stationary process asymptotically with bounded fourth moment.
In addition, $\{d_k\}$ is ergodic since each element is ergodic.	
For the second statement of this lemma, it is not difficult to obtain that there exists a sequence of shift transformations $\{T_k\}$ such that $\lim\limits_{k\to\infty}T_k=T$ and $T$ is measure-preserving.
Then following the proof-line (accompanied with a measure-preserving transformation) for the result that any measurable function of a stationary ergodic stochastic process is stationary and ergodic \cite{stout1974almost}, we obtain the desired result.
$\hfill\blacksquare$

\noindent\hspace{1em}{\itshape Proof of Lemma \ref{tanh1}:}
From the facts that $z\tanh(z)$ is an even function in $z$ and $\mathbb{E}_{y\sim\frac{1}{2}\mathcal{N}(a,\sigma^2)+\frac{1}{2}\mathcal{N}(-a,\sigma^2)}\left[y\tanh\left(\frac{ay}{\sigma^2}\right)\right]=a$, we can easily obtain the desired results.
$\hfill\blacksquare$

\noindent\hspace{1em}{\itshape Proof of Lemma \ref{le22}:}
From the facts that $\tanh(z)$ and $z\tanh'(z)$ are bounded for $z\in\mathbb{R}$, it follows that $\frac{\partial f(c,x,w)}{\partial x}$ is bounded. Then we have that
$$\begin{aligned}
	&\frac{dF(c,x)}{dx}=\int_{-\infty}^{\infty}\frac{\partial f(c,x,w)}{\partial x}\exp\left(-\frac{w^2}{2\sigma^2}\right)dw\\
	=&\int_{-\infty}^{\infty}\Big[\tanh\left(\frac{c(w+x)}{\sigma^2}\right)-\tanh\left(\frac{c(w-x)}{\sigma^2}\right)\Big]\\
	&\exp\left(-\frac{w^2}{2\sigma^2}\right)dw+\int_{-\infty}^{\infty}\bigg[\frac{c(w+x)}{\sigma^2}\tanh'\left(\frac{c(w+x)}{\sigma^2}\right)\\
	&~~~~~-\frac{c(w-x)}{\sigma^2}\tanh'\left(\frac{c(w-x)}{\sigma^2}\right)\bigg]\exp\left(-\frac{w^2}{2\sigma^2}\right)dw\\
	\buildrel \Delta \over=&L_1(c,x)+L_2(c,x).			
\end{aligned}$$
Since $\tanh(z)$ is odd and increasing, we have for $x>0$, $L_1(c,x)=2\int_{0}^{\infty}\big[\tanh\big(\frac{c(w+x)}{\sigma^2}\big)-\tanh\big(\frac{c(w-x)}{\sigma^2}\big)\big]\exp\big(-\frac{w^2}{2\sigma^2}\big)dw>0.$
Moreover, for $x>0$ and $w>0$, we have $\exp\big(-\frac{(w-x)^2}{2\sigma^2}\big)>\exp\big(-\frac{(w+x)^2}{2\sigma^2}\big).$ By this inequality and the fact $z\tanh'(z)$ is an odd function in $z$, it follows that
$	L_2(c,x)=2\int_{0}^{\infty}\frac{cw}{\sigma^2}\tanh'\big(\frac{cw}{\sigma^2}\big)\big[\exp\big(-\frac{(w-x)^2}{2\sigma^2}\big)-\exp\big(-\frac{(w+x)^2}{2\sigma^2}\big)\big]dw>0	
$.
Lemma \ref{le22} thus be proven. $\hfill\blacksquare$

\noindent\hspace{1em}{\itshape Proof of (\ref{db2}):}
We just provide the proof for the first part of (\ref{db2}), i.e., $\frac{db_2^*(t)}{dt}\leq 0$ if $0\leq b_2^*(t)<\|\beta^*\|$, and the second part can be obtained by following a similar way. 
By (\ref{R11}), (\ref{h22}) and (\ref{varep}), we have $\|v_2(t)(R^{-1}(t)-c_0^{-1}I)f(\beta(t))\|\leq e^{-t}\frac{1}{3}c_0^{-1}(p_1+c_0\bar b)\buildrel \Delta \over=\bar c_2e^{-t}$.
Denote $\bar S(\beta(t))=c_0^{-1}h_2(\beta(t))-\bar c_2 e^{-t}$.
By (\ref{b2*t}), we have $\frac{db_2^{*}(t)}{dt}\leq-\frac{b_1^{*}(t)}{b_1(t)}\bar S(\beta(t)).$
To prove $\frac{db_2^{*}(t)}{dt}\leq0$, it suffices to show that $\bar S(\beta(t))\geq0$ for $t\geq0$.
From Lemma \ref{le22} and its proof, it is clear that there exists a positive constant $\bar m_1$ such that for any $x>0$, $c>0$, we have $0<\frac{dF(c,x)}{dx}\leq\bar m_1$. 
Then by (\ref{h2}), we have 
\begin{equation}\label{hh}
	\begin{aligned}
		&\frac{dh_2(\beta(t))}{db_2^*(t)}\leq\frac{2}{\sqrt{2\pi}\sigma}\int_{a_1(t)>0}\int_{a_2(t)>0}a_2^2(t)\\
		&\sup_{x(t)}\bigg|\frac{dF\big(c(t),x(t)\big)}{dx(t)}\bigg| \tilde g(a_1(t),a_2(t))da_2(t)da_1(t)\leq\bar c_1,
	\end{aligned}	
\end{equation}
where $\bar c_1=\frac{c_0\bar m_1}{2\sqrt{2\pi}\sigma}$.
Besides, we have $|\frac{dh_2(\beta(t))}{dt}|=|\frac{dh_2(\beta(t))}{db_2^*(t)}\times\frac{db_2^*(t)}{dt}|\leq\bar c_1\frac{\|\beta^*\|}{\underline b}|\bar S(\beta(t))|$.
Denote $\bar c_3=c_0^{-1}\bar c_1\frac{\|\beta^*\|}{\underline b}$, we have
$$
\frac{d\bar S(\beta(t))}{dt}=c_0^{-1}\frac{dh_2(\beta(t))}{dt}+\bar c_2 e^{-t}\geq-\bar c_3|\bar S(\beta(t))|+\bar c_2 e^{-t}.$$
Thus we obtain 
\begin{equation}\label{www}
	\frac{d\bar S(\beta(t))}{dt}\big|_{\bar S(\beta(t))=0}\geq\bar c_2 e^{-t}\geq0.
\end{equation}
By (\ref{varep}), we have $\bar S(\beta(0))>0$. Then by (\ref{www}), it is clear that $\bar S(\beta(t))\geq0$ for all $t\geq 0$. $\hfill\blacksquare$

\noindent\hspace{1em}{\itshape Proof of Lemma \ref{regular1}:}
The lemma can be obtained similarly to Lemma \ref{regular}, and proof details are omitted here. $\hfill\blacksquare$

\noindent\hspace{1em}{\itshape Proof of $\beta_k\nrightarrow0$:}
We prove $\beta_k\nrightarrow0$ by contradiction.
Firstly, for $\bar x=[\bar\beta^{\tau}~\text{vec}^{\tau}(\bar R)]^{\tau}\in D_{A}$, from (\ref{p1}) and (\ref{db1}), we have that if $0<\|\bar \beta\|< b_l$ and $\|\bar R-c_0I\|\leq\frac{1}{4}c_0$, then there exists a positive constant $\alpha$ such that $\bar\beta^{\tau}\bar R^{-1}f(\bar \beta)>\|\bar \beta\|\alpha>0.$ 
For any integer $n>0$ and any $\Delta>0$, we define $m(n,\Delta)=\max\{m:\sum\nolimits_{i=n}^{m}\frac{1}{i}\leq\Delta\}$.
If $\beta_k\to0$, then by (\ref{sign}), for sufficiently large $n$, there exist positive constants $\delta_1$, $\delta_2$ and $\delta_3$ such that $\delta_1<\|\beta_n\|<\delta_2<b_l$ and $\|\beta_m-\beta_n\|<\delta_3$ for $m\in[n,m(n,\Delta)]$. It is clear that $\delta_3$ can be sufficiently small when $\Delta_t$ is sufficiently small. 
Besides, by (\ref{RR}), we also have $\|R_n-c_0I\|\leq\frac{1}{4}c_0$ for large $n$. 
Thus we have $\beta_n^{\tau}R_n^{-1}f(\beta_n)>\delta_1\alpha$ for large $n$.
Secondly, from (\ref{recuralg}) and (\ref{ode}), we have $\beta_{m(n,\Delta)+1}=\beta_n+\sum\nolimits_{i=n}^{m(n,\Delta)}\frac{1}{i}Q_1(x_i,\phi_i,y_{i+1})=\beta_n+\sum\nolimits_{i=n}^{m(n,\Delta)}\frac{1}{i}R_n^{-1}f(\beta_n)+L_1(n,\Delta,x_n)+L_2(n,\Delta,x_n)$,	
where $L_1(n,\Delta,x_n)=\sum\nolimits_{i=n}^{m(n,\Delta)}\frac{1}{i}[Q_1(x_n,\phi_i,y_{i+1})- R_n^{-1}f(\beta_n)]$ and $L_2(n,\Delta,x_n)=\sum\nolimits_{i=n}^{m(n,\Delta)}\frac{1}{i}[Q_1(x_i,\phi_i,y_{i+1})-Q_1(x_n,\phi_i,y_{i+1})]$.
By Lemma \ref{regular} and the boundedness of $x_n$ in (\ref{RR}) and (\ref{bound}), we have $L_1(n,\Delta,x_n)\to0$ when $n\to\infty$ and $\Delta$ is small, and $|L_2(n,\Delta,x_n)|\leq\mathcal{R}_1\Delta\max_{m\in[n,m(n,\Delta_t)]}\{|x_m-x_n|\}\leq\mathcal{R}_1\delta_3\Delta$. Thus it is evident that $\|\beta_{m(n,\Delta)+1}\|\geq\|\beta_n\|+\frac{1}{2}\delta_1\alpha\Delta-o(1)-\mathcal{R}_1\delta_3\Delta>\|\beta_n\|>\delta_1$ for sufficiently small $\Delta$ and large $n$. 
Thus there exists a subsequence $\|\beta_{n_k}\|$ with a positive lower bound $\frac{\delta_1}{2}$, which contradicts with $\beta_k\to0$. This completes the proof.$\hfill\blacksquare$

\bibliographystyle{IEEEtran}
\bibliography{references}

\end{document}